\documentclass[english,a4paper,12pt]{article}

\PassOptionsToPackage{pdfpagelabels=false}{hyperref}

\usepackage{amsmath, amsxtra, amsfonts, amssymb, amstext}
\usepackage{amsthm}
\usepackage{booktabs}
\usepackage{nicefrac}
\usepackage{xspace}
\usepackage[noadjust]{cite}
\usepackage{url}
\usepackage{graphics}
\usepackage[usenames,dvipsnames]{xcolor}
\usepackage[colorlinks]{hyperref}
\definecolor{linkblue}{rgb}{0.1,0.1,0.8}
\hypersetup{colorlinks=true,linkcolor=linkblue,filecolor=linkblue,urlcolor=linkblue,citecolor=linkblue}
\usepackage[algo2e,ruled,vlined,linesnumbered]{algorithm2e}
\usepackage{wrapfig}

\usepackage{lmodern}
\allowdisplaybreaks[4]
\newcommand{\oea}{\mbox{$(1 + 1)$~EA}\xspace}

\newcommand{\oplea}{\mbox{$(1+\lambda)$~EA}\xspace}

\newcommand{\onemax}{\textsc{OneMax}\xspace}
\newcommand{\LO}{\textsc{Leading\-Ones}\xspace}
\newcommand{\leadingones}{\LO}
\newcommand{\DLB}{\textsc{Deceptive\-LeadingBlocks}\xspace}
\newcommand{\blockleadingones}{\textsc{Block\-LeadingOnes}\xspace}
\newcommand{\needle}{\textsc{Needle}\xspace}
\newcommand{\royalroad}{\textsc{RoyalRoad}\xspace}

\newcommand{\plateau}{\textsc{Plateau}\xspace}
\newcommand{\jump}{\textsc{Jump}\xspace}

\newcommand{\C}{\ensuremath{\mathbb{C}}}
\newcommand{\R}{\ensuremath{\mathbb{R}}}

\newcommand{\Z}{\ensuremath{\mathbb{Z}}}

\DeclareMathOperator{\mutate}{mutate}

\newcommand{\eps}{\varepsilon}

\newcommand{\assign}{\leftarrow}

\let\originalleft\left
\let\originalright\right
\renewcommand{\left}{\mathopen{}\mathclose\bgroup\originalleft}
\renewcommand{\right}{\aftergroup\egroup\originalright}

\newtheorem{theorem}{Theorem}[section]
\newtheorem{lemma}[theorem]{Lemma}

\newtheorem{definition}[theorem]{Definition}
\newtheorem{proposition}[theorem]{Proposition}
\newtheorem{corollary}[theorem]{Corollary}

\newcommand{\tf}{\tilde{f}}
\newcommand{\fh}{\hat{f}}
\newcommand{\tp}{\tilde{p}}
\newcommand{\ph}{\hat{p}}

\newcommand{\ones}{\mathbf{1^\ell}}
\newcommand{\zeros}{\mathbf{0}}
\newcommand{\norm}[1]{|#1|}

\DeclareMathOperator{\expectation}{E}

\DeclareMathOperator{\argmin}{argmin}

\newcommand{\ds}{\displaystyle}

\title{Fourier Analysis Meets Runtime Analysis: Precise Runtimes on Plateaus}
\author{Benjamin Doerr, Andrew James Kelley}

\begin{document}
	{\sloppy
		
		\maketitle
		\begin{abstract}
			We propose a new method based on discrete Fourier analysis to analyze the time evolutionary algorithms 
			spend on plateaus. This immediately gives a concise proof of the classic estimate of the expected runtime of the $(1+1)$ evolutionary algorithm on the Needle 
			problem due to Garnier, Kallel, and Schoenauer (1999). 
			
			We also use this 
			method to analyze the runtime of the $(1+1)$ evolutionary algorithm on a 
			benchmark consisting of $n/\ell$ plateaus of effective size $2^\ell-1$ which 
			have to be optimized sequentially in a LeadingOnes fashion. 
			
			Using our new method, we  determine 
			the precise expected runtime both for static and fitness-dependent 
			mutation rates. We also determine the asymptotically optimal static and fitness-dependent mutation rates. 
			For $\ell = o(n)$, the optimal static mutation rate is approximately $1.59/n$. The optimal fitness dependent mutation rate, when the first $k$ fitness-relevant bits have been found, is asymptotically $1/(k+1)$. These results, so far only proven for the single-instance problem LeadingOnes, thus hold for a much broader class of problems. We expect similar extensions to be true for other important results on LeadingOnes. We are also optimistic that our Fourier analysis approach can be applied to other plateau problems as well.
		\end{abstract}

		\section{Introduction}
		
		The mathematical runtime analysis of evolutionary algorithms (EAs) and other randomized search heuristics has made great progress in the last twenty years~\cite{AugerD11,DoerrN20,Jansen13,NeumannW10}. Starting with simple algorithms like the \oea on basic benchmark problems like \onemax, the area has quickly advanced to the analysis of complex evolutionary algorithms, estimation-of-distribution algorithms, ant colony optimizers, and many other heuristics, and this for the optimization of combinatorial optimization problems, of multi-objective problems, or in the presence of noise. 
		
		A closer look at the field reveals that we know quite well how to analyze optimization processes where a certain, steady progress is made. Here methods such as the fitness level method~\cite{Wegener01}, the expected weight decrease method~\cite{NeumannW07}, or drift analysis~\cite{HeY01} often allowed researchers to prove tight runtime guarantees, and often in (what now appears as) a straightforward way.
		
		The situation is very different when search heuristics encounter plateaus of constant fitness. Here the above mentioned methods cannot be applied (or only via the use of nontrivial and problem-specific potential functions). A good example for such difficulties is the analysis of Garnier, Kallel, and Schoenauer~\cite{GarnierKS99} on how the \oea optimizes the \needle problem. The \oea is arguable the simplest EA and the \needle problem is clearly the simplest (not easiest) problem with a pleateau -- the whole search space apart from the unique optimum is one large plateau of constant fitness. Despite this purported simplicity, a relatively technical proof was needed to prove the natural result that the \oea takes exponential time to find the optimum of the \needle problem; the paper proves a remarkably tight bound, tight including the leading constant, but no simpler proof is known for when only the asymptotic order of the runtime is sought for. 
		
		The reason for these difficulties is the absence of a natural measure of progress. One would hope that for a problem like \needle the high degree of symmetry could be exploited, but so far this has not been done successfully. The difficulty is as follows. To best exploit the symmetry, one would assume that the algorithm runs indefinitely and one would ignore the selection, that is, the offspring is always accepted even when it has a lower fitness than the parent. Note that this happens only when the current solution is already the optimum, and consequently, the first hitting time of the optimum is the same as for the original optimization process on the \needle problem. Now it is true that at all times the random solution of the \oea is uniformly distributed on the search space, but this still does not easily lead to runtime guarantees due to the dependencies between the iterations. Hence additional mixing time arguments would be necessary.
		
		In this work, we develop a novel approach to this plateau problem that uses discrete Fourier analysis. To the best of our knowledge, this is the first time that discrete Fourier analysis is used in the runtime analysis of a randomized search heuristic. We leave the technical details to Section~\ref{sec:using_fourier_analysis} and state here only that we are optimistic that this approach, while nonstandard in this field, is easy to understand and apply. 
		
		For the optimization process of the \oea (with general mutation rate $p$) on the \needle problem, our approach determines the precise expected runtime to be 
		\[
		E[T] = \sum_{j=1}^n \binom{n}{j} \frac{1}{1 - (1-2p)^j}.
		\]
		This extends the previous best result~\cite{GarnierKS99} to arbitrary mutation rate. Also, not too important but nice to have, our result determines the exact expected runtime, whereas the result of \cite{GarnierKS99} is precise only up to $(1\pm o(1))$ factors. We note that in \cite[Lemma 3.8]{GarnierKS99} also a convergence in distribution was shown. We do not prove any such result. Since the proof of \cite[Lemma 3.8]{GarnierKS99} appears relatively independent of the determination of the expected runtime in~\cite[Lemma 3.7]{GarnierKS99}, we would speculate that similar arguments can also be used to enrich our result with a statement on the distribution of the runtime, but we do not attempt this in this work. 
		
		We apply our method also to a generalization of the \leadingones problem. The \leadingones\ benchmark, first proposed in~\cite{Rudolph97}, is one of the most prominent benchmarks in the theory of evolutionary algorithms. It was the first example to show that, different from what is claimed in~\cite{Muhlenbein92}, not all unimodal problems are solved by the \oea in time $O(n \log n)$~\cite{Droste02}. It was also the first example for which a fitness-dependent mutation rate was proven to be superior to any static mutation rate~\cite{BottcherDN10}. Also for this benchmark, several classic hyperheuristics were shown to not work properly and suitable generalizations were developed~\cite{LissovoiOW17}. 
		
		While it is thus clear that the \leadingones benchmark had a significant impact on the theory of randomized search heuristics, it is also clear that all these results are based on a problem consisting of a single instance per problem size~$n$ (as opposed to other problems regarded in the theory of EAs such as pseudo-Boolean linear functions~\cite{DrosteJW02} and various types of \jump functions~\cite{DrosteJW02,Jansen15,BamburyBD21,DoerrZ21aaai,Witt23} or combinatorial optimization problems such as minimum spanning trees~\cite{NeumannW07}, single-source or all-pairs shortest paths~\cite{ScharnowTW04,DoerrHK12}, or Eulerian cycles~\cite{Neumann08}). This raises the question to what extent the insights gained from the analyses on \leadingones generalize. 
		
		As a first step to answer this question, we propose the \blockleadingones benchmark, which counts from left to right the number of contiguous blocks of fixed length~$\ell$ that only contain ones (mathematically simpler, we have $\blockleadingones(x) = \lfloor \leadingones(x) / \ell \rfloor$ for all $x \in \{0,1\}^n$). This problem can be seen as a \leadingones problem of length $n/\ell$ where each bit position is replaced by a block of length~$\ell$, which contributes a one to the original \leadingones problem if and only if all $\ell$ bits are equal to one (we note that the \royalroad problem~\cite{mitchell92royal} is constructed in this fashion from the \onemax problem). As we shall show in this work, many results previously proven for the \leadingones problem also hold in an analogous fashion for the broader class of \blockleadingones functions.
		
		We note that the \blockleadingones benchmark is essentially equal to the Royal Staircase benchmark introduced in \cite{NimwegenC01}, the difference being an additive term of one (which can be relevant when using fitness-proportionate selection). We also note that the \blockleadingones problem with block length $\ell=2$ has appeared as one of many problems in the experimental study~\cite{DoerrYHWSB20}. Due to the very different settings -- fitness-proportionate selection in~\cite{NimwegenC01} and the small block length, leading to effective plateaus of size three, in~\cite{DoerrYHWSB20} -- we could not distill from these works any greater insights on how simple elitist EAs cope with plateaus of constant fitness.
		
		As is easy to see, the \blockleadingones problem has nontrivial plateaus. We note that already the \leadingones problem has large plateaus of constant fitness, namely the fitness levels 
		\begin{align*}
		L'_i &= \{x \in \{0,1\}^n \mid \leadingones(x)=i\} \\
		&= \{x \in \{0,1\}^n \mid (\forall j \in [1..i]: x_j = 1) \wedge x_{i+1} = 0\},
		\end{align*}
		$i \in [0..n-2]$, but these are not critical as the plateau can be left to an individual with higher fitness by flipping a single bit. For the \blockleadingones problem with block length~$\ell$, the fitness levels 
		\begin{align*}
		L_i &=  \{x \in \{0,1\}^n \mid \blockleadingones(x)=i\} \\
		&= \{x \in \{0,1\}^n \mid (\forall j \in [1..i \ell]: x_j = 1) \wedge \\
		&\quad\quad\quad\quad\quad\quad\quad(\exists j \in [i\ell+1..(i+1)\ell] : x_{j} = 0)\},
		\end{align*}
		$i \in [0..n/\ell-1]$, form nontrivial plateaus in the sense that the closest improving solution might be $\ell$ Hamming steps away. These plateaus pose significant difficulties when attempting a runtime analysis for the \blockleadingones problem, in particular, when aiming for runtime bounds that are tight including the leading constant. So it was these difficulties that led us to find some way to exploit the symmetric nature of the plateau, which was the key behind the Fourier approach used in this paper, and with this approach we derive the following results for the \blockleadingones problem. 
		
		The optimal fitness-dependent mutation rate, $p(k)$, when the first $k$ bits are locked in is asymptotic to $1/(k+1)$ if $\ell$ is constant with respect to $n$. If $\ell = \omega(1)$, then with fitness $m$ and letting $k = m\ell$, we have $\lim_{\ell \to \infty} p(m\ell)/ (\ell^{-1}(\sqrt{1 + 2/m} - 1)) = 1$. When using the optimal fitness-dependent mutation rate, the expected runtime of \blockleadingones is asymptotic to $eb2^\ell n^2/(2\ell)$, where $b = 2^{-\ell -1}\sum_{j=1}^\ell \binom{\ell}{j} \frac{1}{j}$, and for large $\ell$, $b \approx 1/\ell$. 
		% Mention that we also give a formula for the exact expected runtime?
		When using a static mutation rate of $c/n$, the runtime is asymptotic to $b2^\ell n^2 (e^c - 1)c^{-2}\ell^{-1}$, which has the optimal mutation rate of about $1.59/n$, giving a runtime asymptotic to $\alpha b2^\ell n^2/\ell$ for $\alpha \approx 1.54$.
		
		This work is organized as follows. We brief{}ly describe the most relevant previous works in the subsequent section. In Section~\ref{sec:using_fourier_analysis}, we introduce our analysis methods based on Fourier analysis. As a first simple application of this method, we give a new and simple analysis of the runtime of the \oea with arbitrary mutation rate on \needle in Section~\ref{sec:needle_problem}. In Section~\ref{sec:block_leading_ones}, we conduct a runtime analysis of the \oea on \blockleadingones, and determine optimal static and dynamic mutation rates. The conclusion in the last section summarizes our work and points out what could be the next steps in this research direction.

		\section{Previous Works}
		
		We now briefly describe the most relevant previous works which are runtime analyses of evolutionary algorithms. In particular, we mention works (i)~on problems with nontrivial plateaus and (ii)~the \leadingones problem. 
		
		As noticed already in the introduction, the vast majority of mathematical runtime analyses of evolutionary algorithms regard problems without significant plateaus of constant fitness. For these, the typical way the EA progresses is by finding solutions with strictly better fitness, and this allows one to obtain upper bounds (and sometimes also lower bounds) for the expected runtime via adding waiting times for suitable improvements (fitness level method~\cite{Wegener01,Sudholt13,Witt14,LassigS14,DoerrK21gecco} or via analyzing the expected progress with regard to a suitable progress measure (drift analysis~\cite{HeY01,OlivetoW11,DoerrJW12algo,DoerrG13algo,DoerrK21algo,LehreW21}). 
		
		Much less is known on how to analyze evolutionary optimization processes that need to traverse large plateaus of constant fitness. 
		In their seminal work -- the first paper determining runtimes precise including the leading constant and the first runtime analysis for an EA on a problem with a nontrivial plateau -- Garnier, Kallel, and Schoenauer~\cite{GarnierKS99} determine the precise (apart from lower order terms) runtime of the \oea on the \onemax and \needle problems (this lattter result is phrased as optimization without selection, but this is equivalent to saying that one optimizes the \needle function). 
		In this language, the main result for the \needle problem is that the \oea with mutation rate $p=c/n$, $c$ a constant, when initialized with an arbitrary search point different from optimum, finds the optimum of the \needle problem in an expected number of $(1\pm o(1)) 2^n \frac{1}{1 - e^{-c}}$ iterations. 
		This result is proven via a careful and somewhat technical analysis of the Markov chain on the Hamming levels of the hypercube $\{0,1\}^n$. 
		With deeper arguments from the theory of Markov processes, the authors also show that the runtime normalized by the expectation converges in distribution  to an exponential distribution with mean~$1$. The proof of this result~\cite[Lemma~3.8]{GarnierKS99} is sketched only.
		
		With similar, slightly simpler arguments, Wegener and Witt~\cite{WegenerW05} analyzed the runtime of the \oea on monotone polynomials (without making the leading constant of the runtime precise). This result can be used to show that the \oea optimizes Royal Road functions with block size~$d$ in time $\Theta(2^d \frac nd \log(\frac nd +1))$ (implicit in~\cite{WegenerW05}, explicit in~\cite{DoerrSW13foga}).
		
		%! insert here
		The paper \cite{NimwegenC01} on the Royal Staircase function (essentially \blockleadingones) uses a non-elitist genetic algorithm without crossover and with fitness-proportionate selection, but they do mention crossover in their Section 7. Figure 3 of \cite{NimwegenC01} shows, unsurprisingly, that the optimal mutation rate for their GA is less than that of the \oea; this is because  mutation can cause a non-elitist approach to move to lower fitness individuals. For an application of a modified Royal Staircase function to biology, see~\cite{eremeevS21}
		
		The only work, to the best of our knowledge, that explicitly uses mixing time arguments, is the analysis of the \oplea on Royal Road functions~\cite{DoerrK13cec}. Since the main technical challenge there is posed by the large offspring population size, whereas we discuss a single-trajectory heuristic, we give no further details. 
		
		In~\cite{AntipovD21telo}, the $\plateau_k$ problem is defined, which has as plateau a Hamming ball of radius $k$ around the all-ones string (the global optimum). It was shown that the runtime of the \oea on $\plateau_k$ is the size of the plateau times the waiting time to flip between 1 and $k$ bits. In the present paper (after Lemma \ref{lem:simplified_expectation}), we mention that plateaus in \blockleadingones have a completely analogous runtime, despite the  different shape of the plateaus. 
		
		We note in passing that there are three more runtime results for the \plateau problem, all very distant from our work.  In~\cite{Eremeev20}, a runtime analysis of non-elitist population-based algorithms on the \plateau benchmark was conducted. The result on sub-jump functions in~\cite{Doerr21cgajump} and the result on weakly monotonic functions in~\cite{Doerr21tcsUB}, as pointed out in these works, also include \plateau functions as special cases. Since both works do not employ methods specific to plateaus, we do not discuss them further. 
		
		In several analyses, thin plateaus showed up, on which the behavior of the EA can be described via an unbiased Markov chain on an interval of the integers. Such Markov chains are relatively well understood, and various arguments can be used to show that the expected time to reach a desired point in such a chain is at most quadratic in the length of the interval in which this Markov chain lives. 
		Such arguments were used, e.g., to analyze the runtime of the \oea on the problems of computing maximum matchings~\cite{GielW03} or Eulerian cycles~\cite{Neumann08}. Artificial example problems with such one-dimensional plateaus have been analyzed in~\cite{JansenW01,BrockhoffFHKNZ07,FriedrichHN09,FriedrichHN10}.
		
		The \leadingones problem was first proposed by Rudolph~\cite{Rudolph97} as an example of a unimodal function having a runtime higher than the $O(n \log n)$ observed before on \onemax~\cite{Muhlenbein92}. Rudolph proves that the runtime of the \oea on \leadingones is $O(n^2)$, the matching lower bound of $\Omega(n^2)$ was first shown in~\cite{DrosteJW02}. 
		
		After the results for \onemax and \needle in~\cite{GarnierKS99}, the \leadingones problem is the third problem for which precise (that is, including the leading constant) runtime bounds could be shown. In two independent works~\cite{BottcherDN10,Sudholt13}, the runtime of the \oea with mutation rate $0 < p \le \frac 12$ on \leadingones was shown to be exactly $\frac 12 p^{-2} ((1-p)^{-n+1} - (1-p))$. This result implies that the optimal mutation rate is approximately $\frac{1.59}{n}$ (leading to an expected runtime of approximately $0.77n^2$), slightly above the standard recommendation of~$\frac 1n$ (leading to an expected runtime of approximately $0.86n^2$).%, which was shown to be asymptotically optimal in~\cite{GarnierKS99}. 
		
		In~\cite{BottcherDN10}, it was also proven (and for the first time for a classic benchmark) that the optimal mutation rate can change during the optimization process and that exploiting this can lead to constant-factor runtime gains. If the mutation rate $p_i = \frac{1}{i+1}$ is used when the current fitness is equal to~$i$, then the expected runtime reduces to $(e/4)(n^2 + n) \approx 0.68 n^2$. These fitness-dependent mutation rates are optimal. 
		
		The precise understanding of this changing optimal mutation rate motivated several research works on automatically adjusting the mutation strength during the run of an algorithm. Lissovoi, Oliveto, and Warwicker~\cite{LissovoiOW20ecj} used the analysis method of~\cite{BottcherDN10} to analyze the performance of simple hyperheuristics flipping one or two bits. In a sense, the algorithm regarded is again the \oea which has access to the two mutation operators that flip exactly one or exactly two random bits.  They show that the best runtime obtainable in this framework is $\frac 14 (1+\ln 2) n^2 +O(n) \approx 0.42n^2$, which is by a constant factor faster than the $\frac 12 n^2$ runtime resulting from always flipping one bit, which is the classic \emph{randomized local search} heuristic.
		
		Interestingly, they show that several classic hyperheuristics such as \textsc{Permutation}, \textsc{Greedy}, and \textsc{RandomGradient} perform worse and have a runtime of $\frac 12 \ln(3) n^2 + o(n^2) \approx 0.55n^2$. From their proofs, the authors of~\cite{LissovoiOW20ecj} distill a reason for the weak performance of the \textsc{RandomGradient} heuristic and improve it significantly. If the current low-level heuristic (here, the mutation operator) is not changed as soon as no improvement is found, but only if for some longer time $\tau \in \omega(n) \cap o(n \log n)$ no improvement is found, then this generalized \textsc{RandomGradient} heuristic achieves the asymptotically optimal (among all uses of one-bit and two-bit flips) expected runtime of $\frac 14 (1+\ln 2) n^2 +O(n) \approx 0.42n^2$. Similar results were shown for larger numbers of low-level heuristics, we refer to~\cite{LissovoiOW20ecj} for the details. The learning period $\tau$ can be chosen in a self-adjusting fashion~\cite{DoerrLOW18}, rendering the hyperheuristic an essentially parameter-free algorithm.
		
		Also with the standard bit-wise mutation operator automated parameter choices have been investigated. Following an experimental study~\cite{DoerrW18}, the runtime analysis~\cite{DoerrDL21} studies the effect of adjusting the mutation rate of the standard bit-wise mutation operator in the \oea via a one-fifth rule. More precisely, it shows that when using a $1/s$-rule and a cautious multiplicative update factor $F = 1+o(1)$, this self-adjusting algorithm manages to keep the mutation rate at $(1 \pm o(1)) \frac{\ln(s)}{f(x)}$, where $f(x)$ is the current fitness value. This is only by a constant factor of $\ln(s)$ off the optimal fitness-dependent mutation rate determined in~\cite{BottcherDN10}. In particular, a $1/e$-success rule determines the asymptotically optimal mutation rates and leads to the asymptotically optimal expected runtime for the \oea with bit-wise mutation. 
		
		These results show that significant insights were gained from studying the \leadingones benchmark. It appears likely that similar results hold more broadly than just for this one function. However, no such result exists so far. The most likely reason for this is the lack of other benchmarks in which a typical optimization process shows some steady progress towards the optimum. We note that when optimizing \onemax, the most prominent benchmark, almost all of the optimization time is spent on the last lower-order fraction of the fitness levels, hence often the parameters optimal for these are also asymptotically optimal for the whole process. Even more extreme is the situation for the \jump benchmark, where the runtime is dominated by the time taken by the last fitness improvement and hence this alone determines the asymptotically optimal mutation rate~\cite{DoerrLMN17}.
		
		We note that another variant of the \leadingones benchmark, the \DLB problem, was proposed in~\cite{LehreN19foga}. Here also blocks, always of length two, have to be optimized in a sequential fashion, but each block is deceiving in the sense that a block value of $11$ gives the best fitness contribution, one of $00$ the second best, and the other two the worst. We believe that this problem also could be an interesting object of study for the topics studied on \leadingones so far. However, with the larger number of local optima, this might rather be a benchmark to study how randomized search heuristics cope with local optima. In fact, in~\cite{WangZD21} it was shown that the \oea has a significantly worse performance on this problem than the Metropolis algorithm~\cite{MetropolisRRTT53} and the significance-based estimation-of-distribution algorithm~\cite{DoerrK20tec}. For this reason, we expect \blockleadingones to be a more interesting object of study to understand how EAs cope with plateaus of constant fitness. 
		
		Fourier analysis has been used before in analyzing EAs. The authors of \cite{chicanoSWDA15} use it to calculate the moments of the fitness distribution of offspring after applying mutation. For real-valued functions defined on $\{1, 2, \ldots, q\}^n$ that have bounded epistasis, the moments of their frequency distribution were calculated in \cite{suttonCW13}. A connection between the Fourier transform and genetic algorithms was also shown in \cite{vose1998}. See also \cite{roweVW04}, where it is shown that the usefulness of a Fourier approach intimately depends on having a mutation operator that comes from an abelian group (instead of a non-abelian group, such as the set of all permutations on $n\geq 3$ letters). However, to the best of our understanding, Fourier analysis has not been used so far to analyze the runtime of an EA.

		\section{Using Fourier Analysis}
		\label{sec:using_fourier_analysis}
		
		This paper only assumes the reader knows a few elementary facts about what in mathematics is called a \emph{group},
		more specifically what an \emph{abelian group} is (i.e.\ a commutative group). %Any basic abstract algebra text will define what a group is.
		All groups considered in this paper are abelian.
		
		Let $X_t$ be a random walk on the group $G$ with identity $\zeros$. For $g \in G$, we define $\expectation_\zeros[\tau_g]$ as the expectation
		of the hitting time of the element $g$ given that we start at $\zeros$:
		\[
		\expectation_\zeros[\tau_g] = \expectation[\min \{t \mid X_t = g, \; \text{given } X_0 = \zeros\}].
		\]
		
		% Finally, we conclude by defining the $(1 + 1)$ EA:

		%\section{Using Fourier Analysis}
		%\label{sec:using_fourier_analysis}
		
		In this section, we first describe the relevant random walk and then briefly review a few facts about groups, homomorphisms, and Fourier analysis. We then state and use our main tool: Theorem~\ref{thm:representations}, used to prove Theorem~\ref{thm:exact_expectation}.
		
		Let $\mu$ be a probability distribution on a group $G$. Then $\mu$
		defines a random walk on $G$, where for $u, w \in G$,
		the random walk goes
		from $u$ to $u + w$ with probability $\mu(w)$.
		The random walk we are interested in is to define $\mu$ as 
		follows.
		For $w \in G = \Z_2^\ell$, we have %TODO: connect this to the (1 + 1) EA by pointing out this is what you get from it
		\[
		\mu(w) = p^{\norm{w}}(1 - p)^{\ell - \norm{w}},
		\]
		where $p$ is some fixed probability with $p \in (0, 1)$, 
		and where $\norm{w}$ is the 1 norm of $v$ (i.e.\ $\norm{w} = \sum_{i=1}^\ell w_i$). 
		%Then the random walk goes from $u$ to $u + w$ with probability $\mu(w)$. 
		Notice that the resulting random walk is precisely the random
		walk where each bit is flipped independently with probability $p$, which is what is happening in the
		evolutionary algorithm considered in this paper.
		
		Recall that the order of an element $g$ of a group is the smallest positive integer $n$ such that $g^n = 1$, if the group is written multiplicatively (and replace $g^n=1$ with $g  +\cdots +g = 0$, with $n$ $g$'s, if it is written additively).
		
		We next briefly review \emph{homomorphisms}. 
		Let $G$ be an (abelian) group written additively, and let $H$ be a group
		written multiplicatively. Then a homomorphism from $G$ to $H$ is just a function $\varphi : G \to H$ such that
		\[
		\varphi(a + b) = \varphi(a)\varphi(b) \text{\quad for all } a, b \in G.
		\]
		For instance every exponential function $\varphi(x) = b^x$, with $b > 0$, is a homomorphism from the additive group of
		all real numbers $(\R, +)$ to the multiplicative
		group of all positive real numbers: $(\R_{> 0}, \cdot)$.
		
		A \emph{character} $\varphi : G \to \C^*$ of an abelian group $G$ is a homomorphism from $G$ to $\C^*$, the multiplicative group
		of nonzero complex numbers. 
		% Recall that if $g \in G$ has order 2 (which means $g + g = \zeros$ and $g \neq \zeros$, where $\zeros$ is the identity of $G$), then $\varphi(g)$ has order dividing 2. 
		If $g \in G$ has order 2 or 1, then $\varphi(g)$ is a complex number whose square is 1, in which case
		$\varphi(g) \in \{1, -1\}$.
		
		We will only consider characters of abelian groups $G$ in which each element has order 2 or 1. So in this paper,
		a character %\footnote{There is another notion of a character commonly used, namely as the trace of a representation
		%of the group, but for finite abelian groups, there is no difference between that and what is defined here if we 
		%restrict ourselves to irreducible representations. This is because every irreducible representation of a finite abelian group
		%is one dimensional. Finally, note that there is no real distinction between a character of a finite abelian group
		%and an irreducible representation of a finite abelian group.} 
		of $G$ is just a homomorphism
		\[
		\varphi : G \to \{1, -1 \}.
		\]
		In fact, $G$ will be the group $\Z_2^{\ell}$, the Cartesian product of $\Z_2 = \{0, 1 \}$ with itself $\ell$
		times (where $\ell$ is some positive integer), where addition is modulo 2. 
		For $x \in \Z_2^{\ell}$, we denote its $i$th bit by $x_i$ (1-indexed).
		Every character of $\Z_2^{\ell}$ is of the form
		$\rho_v$, where $v \in \Z_2^{\ell}$ and where we define $\rho_v(w)$ for $w \in \Z_2^{\ell}$ by
		\[
		\rho_v(w) = (-1)^{\sum_{i=1}^\ell v_i w_i}.
		\]
		These are the $2^\ell$ characters  of $\Z_2^{\ell}$, one for each $v$.
		
		We denote by $\hat{G}$ the set of all characters of $G$.
		For any function $\mu$ defined on an abelian group $G$ (taking on real values), we have that the \emph{Fourier transform} of $\mu$ is
		a function from $\hat{G}$ to $\C$, given by
		\[ 
		\hat{\mu}(\rho) = \sum_{w \in G} \mu(w) \overline{\rho(w)},
		\]
		where $\overline{z}$ is the complex conjugate of $z$. When $G$ is the group $\Z_2^{\ell}$, any character $\rho$ takes on only real values, and hence,
		\[ % TODO: add a label and tag to this
		\hat{\mu}(\rho) = \sum_{w \in G} \mu(w) \rho(w).
		\]
		For additional background on Fourier analysis on finite abelian groups, see for instance \cite{Garrett12}.

		The following is a special case of Theorem 3.1 from \cite{Zhang23}.
		\begin{theorem}
			\label{thm:representations}
			Let $G$ be the abelian group $\Z_2^\ell$ with $2^\ell = m$, and let $\rho_0, \rho_1, \ldots, \rho_{m-1}$ be the 
			characters of $G$, with $\rho_0$ being the trivial character $\rho_0 : G \to \{1\}$. Let $\mu$ be a probability
			distribution on $G$, and consider the random walk on $G$ generated by $\mu$ (where the walk goes from $g$
			to $g+h$ with probability $\mu(h)$). Then 
			\[
			\expectation_\zeros[\tau_g] = \sum_{i=1}^{m-1} \frac{1 - \rho_i(g)}{1 - \hat{\mu}(\rho_i)}.
			\]
			%where $\expectation_\zeros[\tau_g]$ is the expectation of the hitting time of the element $g$ given that the random
			%walk starts at the identity element $\zeros$.
		\end{theorem}
		
		To exploit Theorem~\ref{thm:representations}, we need to understand the Fourier transform $\hat{\mu}$ of $\mu$.
		The following lemma is very similar to Proposition 3.6 of \cite{vose1998}.
		
		\begin{lemma}
			\label{lemma:fourier_transform_of_mu}
			Let $v \in G = \Z_2^{\ell}$.  Let $\hat{\mu}$, $\mu$, and $\rho_v$ be as in earlier this section.
			Then
			\[
			\hat{\mu}(\rho_v) = (1 - 2p)^{\norm{v}}.
			\]
		\end{lemma}
		\begin{proof}
			Let $\norm{v} = k$.
			Then
			\[
			\begin{aligned}
			\hat{\mu}(\rho_v) &= \sum_{w \in \Z_2^\ell}  \mu(w) \overline{\rho_v(w)} 
			= \sum_{w \in \Z_2^\ell}  \mu(w) \rho_v(w) \\
			&= \sum_{w \in \Z_2^\ell} p^{\norm{w}} (1-p)^{\ell - \norm{w}} (-1)^{\sum_{i=1}^\ell v_i w_i}.
			\end{aligned}
			\]
			Since this sum ranges over all $w \in \Z_2^\ell$, by symmetry we may assume that it is the first $k$ bits of $v$
			that are ones, with the rest being 0. We have then that
			\[
			\hat{\mu}(\rho_v) = \sum_{w \in \Z_2^\ell} p^{\norm{w}} (1-p)^{\ell - \norm{w}} (-1)^{\sum_{i=1}^k w_i}.
			\]
			Writing $w$ as the concatenation of a bitstring $w_a$ of length $k$ and $w_b$
			of length $\ell - k$, we find that
			\[
			\begin{aligned}
			\hat{\mu}(\rho_v) &= \!\! \sum_{w_a \in \Z_2^k}  \sum_{w_b \in \Z_2^{\ell - k}} \! p^{\norm{w_a} + 
				\norm{w_b}}(1 - p)^{k + \ell - k - (\norm{w_a} + \norm{w_b})} (-1)^{\norm{w_a}} \\
			&= \!\! \sum_{w_a \in \Z_2^k}  \!  p^{\norm{w_a}}(1 - p)^{k - \norm{w_a}} (-1)^{\norm{w_a}} B,  \\
			\end{aligned}
			\]
			where
			\[
			B = \sum_{w_b \in \Z_2^{\ell - k}} p^{\norm{w_b}} (1-p)^{\ell - k - \norm{w_b}}.
			\]
			By the binomial theorem we have 
			\[
			B = \sum_{j=0}^{\ell - k} \binom{\ell - k}{j}p^j (1-p)^{\ell - k - j} = (p + 1 - p)^{\ell - k} = 1.
			\]
			Using the binomial theorem a second time, we compute
			\[
			\begin{aligned}
			\hat{\mu}(\rho_v) &= \sum_{w_a \in \Z_2^k}  p^{\norm{w_a}}(1 - p)^{k - \norm{w_a}} (-1)^{\norm{w_a}} \\
			&= \sum_{w_a \in \Z_2^k}  (-p)^{\norm{w_a}}(1 - p)^{k - \norm{w_a}} \\
			&= \sum_{j=0}^k \binom{k}{j} (-p)^j (1-p)^{k - j} \\
			&= (1 - 2p)^k,
			\end{aligned}
			\]
			which proves the result.
		\end{proof}

		The heart of the following theorem is Theorem \ref{thm:representations}.
		\begin{theorem}
			\label{thm:exact_expectation}
			The exact expected hitting time of $\ones$ (the all 1's string) given that we start from a uniform distribution on $\Z_2^\ell$ is
			\[
			\expectation[T] = \sum_{j=1}^\ell \binom{\ell}{j} \frac{1}{1 - (1-2p)^j}.
			\]
		\end{theorem}
		\begin{proof}
			We will see that this follows from Theorem \ref{thm:representations} and Lemma \ref{lemma:fourier_transform_of_mu}. 
			First, note that for all $a, b, c \in G$ we have $\expectation_a[\tau_{b}] = \expectation_{a+c}[\tau_{b+c}]$.
			Therefore, 
			\[
			\expectation[T] =  \frac{1}{2^\ell} \sum_{v \in \Z_2^\ell} \expectation_v[\tau_{\ones}] 
			= \frac{1}{2^\ell} \sum_{v \in \Z_2^\ell} \expectation_{v+v}[\tau_{\ones+v}]
			= \frac{1}{2^\ell} \sum_{w \in \Z_2^\ell} \expectation_\zeros[\tau_{w}],
			\] 
			and notice $\expectation_\zeros[\tau_\zeros] = 0$. By Theorem \ref{thm:representations}  and Lemma \ref{lemma:fourier_transform_of_mu},
			\[\begin{aligned}
			\expectation[T] &= \frac{1}{2^\ell} \sum_{\substack{w \in \Z_2^\ell \\ w \neq \zeros }} \expectation_\zeros[\tau_{w}] \\
			&= \frac{1}{2^\ell} \sum_{\substack{w \in \Z_2^\ell \\ w \neq \zeros }} 
			\sum_{\substack{v \in \Z_2^\ell \\ v \neq \zeros} } \frac{1 - \rho_v(w)}{1 - \hat{\mu}(\rho_v)} \\
			&= \frac{1}{2^\ell} \sum_{w } \sum_{v} 
			\frac{1 - (-1)^{\sum_{i=1}^\ell v_i w_i}}{1  - (1 - 2p)^{\norm{v}}} \\
			&= \frac{1}{2^\ell} \sum_{v} \sum_{w} 
			\frac{1 - (-1)^{\sum_{i=1}^\ell v_i w_i}}{1  - (1 - 2p)^{\norm{v}}}. \\ 
			\end{aligned}
			\]
			Pulling out what does not depend on $w$ gives
			\[
			\expectation[T] = \frac{1}{2^\ell} \sum_{\substack{v \in \Z_2^\ell \\ v \neq \zeros} }  \frac{1}{1  - (1 - 2p)^{\norm{v}}}
			\sum_{\substack{w \in \Z_2^\ell \\ w \neq \zeros }}   \left( 1 - (-1)^{\sum_{i=1}^\ell v_i w_i} \right).
			\]
			A summand for $w$ is nonzero precisely if $v$ and $w$ share an odd number of 1's. There are 
			$2^{\norm{v} - 1} \cdot 2^{\ell - \norm{v}}$ such summands each equal to 2. We thus obtain
			\[\begin{aligned}
			\expectation[T] &= \frac{1}{2^\ell} \sum_{\substack{v \in \Z_2^\ell \\ v \neq \zeros} }  \frac{1}{1  - (1 - 2p)^{\norm{v}}}
			\cdot 2^{\norm{v} - 1} \cdot 2^{\ell - \norm{v}} \cdot 2 \\
			&=  \frac{1}{2^\ell} \sum_{\substack{v \in \Z_2^\ell \\ v \neq \zeros} }  \frac{2^\ell}{1  - (1 - 2p)^{\norm{v}}} \\
			&= \sum_{\substack{v \in \Z_2^\ell \\ v \neq \zeros} }  \frac{1}{1  - (1 - 2p)^{\norm{v}}} \\
			&= \sum_{j=1}^\ell \binom{\ell}{j} \frac{1}{1 - (1-2p)^j}.
			\end{aligned}
			\]
		\end{proof}
		
		\begin{proposition}
			\label{prop:exact_expectation_starting_away_from_ones}
			Let $T$ be as in Theorem \ref{thm:exact_expectation}, and 
			let $T'$ be the expected hitting time of $\ones$ (the all 1's string) given that we start from a uniform distribution on
			the $2^\ell - 1$ elements of $\Z_2^\ell - \{\ones\}$. Then
			\[
			\expectation[T']  =   \frac{2^\ell}{2^\ell - 1} \expectation[T].
			\]
			Therefore, 
			\[
			\expectation[T] = \frac{2^\ell}{2^\ell - 1} \sum_{j=1}^\ell \binom{\ell}{j} \frac{1}{1 - (1-2p)^j}.
			\]
		\end{proposition}
		\begin{proof}
			This follows from the proof of Theorem \ref{thm:exact_expectation} together with the fact that
			\[
			\expectation[T'] = \frac{1}{2^\ell - 1} \sum_{\substack{w \in \Z_2^\ell \\ w \neq \zeros }} \expectation_\zeros[\tau_{w}].
			\]
		\end{proof}

		\section{Analysis of the Needle Problem}
        \label{sec:needle_problem}
		
		In this section, we sketch a simple proof of the classic result on the runtime of the \oea
		on the \needle problem, proven in \cite{GarnierKS99}. There, the mutation rate $p$ is replaced with $c/\ell$, where $c > 0$ is constant. 
		Corollary~\ref{cor:prior_result} 
		is a consequence of Theorem \ref{thm:exact_expectation}, and note that Corollary \ref{cor:prior_result} implies the part of 
		Proposition~3.1 in \cite{GarnierKS99} about $\expectation[T]$, which is the runtime on the \needle problem for the $(1+1)$ EA with 
		bit mutation
		$c/\ell$. The rest of Proposition~3.1 in \cite{GarnierKS99} (which is about random local search) can be proven with modified versions of Lemma~\ref{lemma:fourier_transform_of_mu} and Theorem~\ref{thm:exact_expectation} and using Lemma~\ref{lem:from_Marcin}.

  \begin{algorithm2e}
  	Choose $x \in \{0,1\}^n$ uniformly at random\;
  	\For{$t=1,2,3,\ldots$}{
  		$y \assign \mutate(x)$\;
  		%    \For{$i \in [1..n]$}{with probability~$\frac 1n$ do $y_i \assign 1 - y_i$\;}
  		\lIf{$f(y)\geq f(x)$}{$x \assign y$}
  	}
  	\caption{The  \textbf{(1+1) Evolutionary Algorithm} for maximizing a given objective function~$f\colon\{0,1\}^n\to\mathbb{R}$. 
  		%When the mutation operator returns a random Hamming neighbor of~$x$, that is, flips a random bit in~$x$, then this algorithm is the RLS heuristic. 
  		Here, the mutation operator is to flip each bit independently with probability $p$.
  		The classic \oea uses a mutation rate of $p=\frac 1n$.}
  	\label{algo}
  	\label{alg:oea}
  \end{algorithm2e}
		\begin{corollary}
			\label{cor:prior_result}
			Fix a constant $c > 0$. Then
			\[
			\lim_{\ell \to \infty} 2^{-\ell} \sum_{j=1}^\ell \binom{\ell}{j} \frac{1}{1 - (1 -\frac{c}{\ell/2})^j} =  \frac{1}{1 - e^{-c}}.
			\]
		\end{corollary}
		We first prove a lemma (and then use two more stated afterwards).

		\begin{lemma}
			\label{lemma:middle_binomial_coefficients}
			For all $\alpha \in (0, 1)$, we have
			\[
			\lim_{\ell \to \infty} \sum_{j = (1-\alpha)\ell/2}^{(1+\alpha)\ell/2} \binom{\ell}{j} 2^{-\ell} = 1,
			\]
			where we assume $j$ is integral by appropriately taking the ceiling or floor.
		\end{lemma}
		\begin{proof}
			This follows easily from the Central Limit Theorem, and alternatively, it follows easily from the
			additive Chernoff bound (Theorem 1.10.7 from \cite{Doerr20bookchapter}).
			
			Indeed, let $X_\ell$ be a binomial random variable that is a sum of $\ell$ (independent) Bernoulli trials, each with probability
			1/2 of success. Then $\Pr(X_\ell = j) = \binom{\ell}{j} 2^{-\ell}$. Let $\Phi$ be the standard normal distribution, and let
			$a \leq b$. Then the Central Limit Theorem gives us this:
			\[
			\lim_{\ell \to \infty} \Pr\left(\frac{\ell}{2} - a\frac{\sqrt{\ell}}{2} \leq X_\ell \leq \frac{\ell}{2} + b\frac{\sqrt{\ell}}{2} \right)
			= \int_a^b \Phi(x) dx.
			\]
			Pick any $\alpha \in (0, 1)$. We have the following:
			\[
			\sum_{j = (1-\alpha)\ell/2}^{j = (1+\alpha)\ell/2} \binom{\ell}{j} 2^{-\ell} = 
			\Pr \left(\frac{\ell}{2} - \frac{\alpha \ell}{2} \leq X_\ell \leq \frac{\ell}{2} + \frac{\alpha \ell}{2}\right).
			\]
			The present result follows because  $\ell$ (and hence $\alpha \ell/2$) is $\omega(\sqrt{\ell})$.
		\end{proof}
		
		\begin{proof}[Proof of Corollary \ref{cor:prior_result}]
			Let $\ds{g(\ell, j) = \frac{1}{1 - (1-2c/\ell)^j}}$. Note that for fixed $\ell$, we have that $g(\ell, j)$ is a decreasing function of $j$.
			For $a < b$, define $S_\ell(a, b)$ as
			\[
			S_\ell(a, b) = 2^{-\ell} \sum_{j=a}^b \binom{\ell}{j} g(\ell, j).
			\]
			Let $\alpha \in (0, 1)$. 
			Since $g(\ell, j)$ is decreasing in $j$, by the symmetry of the binomial coefficients, a consequence of (\ref{step1}) below is that
			$\ds{\lim_{\ell \to \infty} S_\ell\left((1+\alpha)\frac{\ell}{2}, \ell\right) = 0}$.
			The present result will follow once we have proved the following two things:
			\[\tag{1}\label{step1}
			\lim_{\ell \to \infty} S_\ell\left(0, (1-\alpha)\frac{\ell}{2}\right) = 0
			\]
			and 
			\[\tag{2}\label{step2}
			\frac{1}{1 - e^{-c(1+\alpha)}} \leq \lim_{\ell \to \infty} S_\ell\left((1-\alpha)\frac{\ell}{2}, (1+\alpha)\frac{\ell}{2} \right)
			\leq \frac{1}{1 - e^{-c(1-\alpha)}}.
			\]
			
			For a positive integer $a$, let $f(a) = \sum_{j=0}^a \binom{\ell}{j}$. By
			Lemma \ref{lem:bounding_sum_of_first_binomial_coefficients}, we have
			\[
			\begin{aligned}
			f\left((1-\alpha)\frac{\ell}{2}\right) &\leq \binom{\ell}{(1-\alpha)\frac{\ell}{2}} \frac{\ell - ((1-\alpha)\ell/2 - 1)}{\ell - (2(1-\alpha)\ell/2 - 1)} \\
			& \leq \binom{\ell}{(1-\alpha)\frac{\ell}{2}} \frac{1 + \alpha + 2/\ell}{2\alpha + 2/\ell}.
			\end{aligned}
			\]
			By Lemma \ref{lem:approximate_of_binomial_coeff}, we have that $\ds{\binom{\ell}{(1-\alpha)\frac{\ell}{2}} }$ equals
			\[
			\begin{aligned}
			&(1 + o(1)) \sqrt{\frac{\ell}{2\pi (1-\alpha)\frac{\ell}{2}(1+\alpha) \frac{\ell}{2}}}
			\left( \frac{\ell}{(1-\alpha)\ell/2}\right)^{(1-\alpha)\ell/2}  \left( \frac{\ell}{(1+\alpha)\ell/2}\right)^{(1+\alpha)\ell/2} \\
			&=(1 + o(1)) 2^\ell c^{\ell/2} \sqrt{\frac{2}{\pi (1 - \alpha^2)\ell}} ,  \text{\quad where } c = \frac{1}{(1-\alpha)^{1-\alpha} (1+\alpha)^{1+\alpha}}.
			\end{aligned}
			\]
			Basic Calculus shows that $c \in (0, 1)$.
			
			Since $\ds{g(\ell, j) \leq g(\ell, 0) = \frac{\ell}{2c}}$, we have 
			\[
			\begin{aligned}
			S_\ell(0, (1-\alpha)\ell/2) &\leq 2^{-\ell} \frac{\ell}{2c} f\left((1-\alpha)\frac{\ell}{2}\right) \\
			& \leq \frac{\ell}{2c} (1 + o(1)) c^{\ell/2} \sqrt{\frac{2}{\pi (1 - \alpha^2)\ell}} \frac{1 + \alpha + 2/\ell}{2\alpha + 2/\ell},
			\end{aligned} 
			\]
			which approaches 0 as $\ell \to \infty$ because of the exponential $c^{\ell/2}$. This proves (\ref{step1}) above.
			
			Let $a  = (1-\alpha)\ell/2$ and $b = (1+\alpha)\ell/2$, and let $j \in (a, b)$. We have
			\[\tag{3}\label{inequalities_bounding_g}
			g(\ell, b) \leq g(\ell, j) \leq g(\ell, a),
			\]
			and we also have
			\[
			\begin{aligned}
			\lim_{\ell \to \infty} g(\ell, b) &= \lim_{\ell \to \infty} \frac{1}{1-(1-\frac{c}{\ell/2})^{(1+\alpha)\ell/2}} = \frac{1}{1 - e^{-c(1+\alpha)}} \\
			\lim_{\ell \to \infty} g(\ell, a) &= \lim_{\ell \to \infty} \frac{1}{1-(1-\frac{c}{\ell/2})^{(1-\alpha)\ell/2}} = \frac{1}{1 - e^{-c(1-\alpha)}}.
			\end{aligned}
			\] 
			By (\ref{inequalities_bounding_g}), we have
			\[
			\sum_{j=a}^b \binom{\ell}{j} 2^{-\ell} g(\ell, b) \leq S_\ell(a, b) \leq \sum_{j=a}^b \binom{\ell}{j} 2^{-\ell} g(\ell, a).
			\]
			Then (\ref{step2}) follows from the above two limits of $g$ and Lemma \ref{lemma:middle_binomial_coefficients}.
		\end{proof}
		
		%Here is an edited version of what I wrote earlier on Corollary \ref{cor:prior_result} (which contains some of the ideas of its proof):
		
		%Let  $\displaystyle{g(\ell, j) = \frac{1}{1 - (1 -\frac{c}{\ell/2})^j}}$. Of course, for fixed $\ell$, we see that $g(\ell, j)$ is a decreasing
		%function of $j$ and that $g(\ell,1) = \frac{\ell}{2c}$. Also, for $j \approx \ell/2$ we have 
		%$g(\ell, j) \approx \frac{1}{1 - e^{-c}}$.
		
		%For $\delta \in [1/2, 1)$, I see that the first $n^\delta$ terms in the sum in the goal don't contribute much
		%to the total. To see this, I used the elementary inequality $\displaystyle{\binom{\ell}{k} \leq \frac{\ell^k}{k!}}$
		
		%for $f(n, k)$,  where we define $f(n,k) = \sum_{j=0}^k \binom{n}{j}$:
		\begin{lemma}
			\label{lem:bounding_sum_of_first_binomial_coefficients}
			Let $k$ and $\ell$ be positive integers with $k < \ell/2$. Then
			\[
			\sum_{j=0}^k \binom{\ell}{j} \leq \binom{\ell}{k} \frac{\ell - (k-1)}{\ell - (2k - 1)}.
			\]
		\end{lemma}
		This lemma is elementary and well known. See 
		mathoverflow.\footnote{\url{https://mathoverflow.net/questions/17202/sum-of-the-first-k-binomial-coefficients-for-fixed-n}}

		\begin{lemma}
			\label{lem:approximate_of_binomial_coeff}
			Suppose $k$ and $n-k$ approach infinity as $n \to \infty$. Then
			\[
			\binom{n}{k}  = (1 + o(1)) \sqrt{\frac{n}{2\pi k (n-k)}} \left(\frac{n}{k}\right)^k \left(\frac{n}{n-k} \right)^{n -k}.
			\]
		\end{lemma}
		This follows from Stirling's formula.
		
		\section{The \blockleadingones Problem}
		\label{sec:block_leading_ones}
		%\section{Optimal static and adaptive mutation rate}
		
		In this section, we regard the \blockleadingones problem as a natural extension of the \leadingones problem. Using our general method, we prove precise bounds for the runtime of the \oea, both with static and dynamic mutation rates, on this benchmark. We use these to determine the asymptotically optimal static and dynamic mutation rates. 
		
		\subsection{Definition of the \blockleadingones Problem}
		\label{sec:def_of_blockleadingones}
		
		The \blockleadingones problem consists of a sequence of blocks of $\ell$ bits each, which have to be optimized sequentially in a \leadingones fashion; the next block is only relevant for the fitness if all previous blocks have already been optimized. Each block is a plateau contributing constant zero fitness except when all bits are one, when it contributes one to the fitness (if all previous blocks are optimized). 
		
		More formally, let the \emph{block length} $\ell$ be an integer that divides the problem size~$n$. 
		Let $x \in \{0,1\}^n$. Define $y_i$ to be 1 if all the bits of $x$ in the $i$th block are 1 and 0 otherwise; in other words, let $y_i = \prod_{j=(i-1)\ell+1}^{i\ell} x_j$.
		Then the fitness of $x$ is defined by
		\[
		\blockleadingones(x) = \sum_{m=1}^{n/\ell} \prod_{i=1}^{m} y_i.
		\]
		This is equivalent to the definition given in the introduction:
		$\blockleadingones(x) = \lfloor \leadingones(x) / \ell \rfloor$.
		%We also note that $\blockleadingones(x) = \sum_{m=1}^{n/\ell} \prod_{i=1}^{m\ell} x_i$. 
		For $\ell=1$, we obtain the classic \leadingones problem.
		
		In what follows, we allow $\ell$ to depend on~$n$. We  assume  $\ell = o(n)$ from Theorem~\ref{thm:simple_runtime} on. The case $\ell = n$ is the classic \needle problem dealt with earlier. Other choices for $\ell = \Theta(n)$ are ignored because in the proof of Theorem~\ref{thm:simple_runtime}, we want to exploit that $(1-c/n)^\ell \to 1$ as $n \to \infty$ when $c$ is constant and $\ell = o(n)$. 
		%In this theorem, the assumption really is necessary. 
		%In Theorem~\ref{thm:optimal_mutation_rate}, the assumption $\ell = o(n)$ is not a very substantive restriction, since the optimal fitness-dependent mutation rate doesn't depend on $n$.
		
		\subsection{Our Results}
		\label{sec:our_results}
		
		\begin{theorem}
			\label{thm:exact_fitness_dep_runtime}
			Let $T$ be the runtime of the \oea on the \blockleadingones problem when using mutation rate $p_m$ when the current fitness is $m$.
			Then
			\[
			\expectation[T] = \sum_{m=0}^{n/\ell-1} \left(\frac{1}{(1-p_m)^{m\ell}} 
			\left(\sum_{j=1}^\ell \binom{\ell}{j} \frac{1}{1 - (1-2p_m)^j}\right)\right).
			\]
			When using a static mutation rate of $p$, this simplifies to
			\[
			\frac{(1-p)^{-n+\ell} - (1-p)^\ell}{1-(1-p)^\ell} \sum_{j=1}^\ell \binom{\ell}{j} \frac{1}{1 - (1-2p)^j}.
			\]
		\end{theorem}
		Note that the proof of Theorem~\ref{thm:exact_fitness_dep_runtime} applies even for $\ell = 1$ and so gives a new proof of Theorem 3 in  \cite{BottcherDN10} that says the expected runtime on \leadingones is
		\[
		\frac{1}{2p^2}((1-p)^{-n-1} - (1-p)).
		\]
		
		\begin{theorem}
			\label{thm:simple_runtime}
			Let $T$ be the runtime of the \oea with static mutation rate $c/n$ on the \blockleadingones problem. Here, $c > 0$ is constant. Define $b$ as $2^{-\ell -1}\sum_{j=1}^\ell \binom{\ell}{j} \frac{1}{j}$, and let 
			$a = 2^{-1} - 2^{-\ell-1} - b$. Assume $\ell = o(n)$.
			Then
			\[
			\expectation[T] = (1 + o(1))\frac{n 2^{\ell}}{\ell}\left(\frac{bn}{c^2} + \frac{a}{c} \right)\left(e^c - 1 \right).
			\]
		\end{theorem}
		Note that an approximation for  $\sum_{j=1}^\ell \binom{\ell}{j} \frac{1}{j}$ and hence $b$ is given in Lemma~\ref{lem:from_Marcin}. To get the expected runtime for the standard bit mutation $1/n$, just plug in $c=1$ into Theorem \ref{thm:simple_runtime}. Also, note that Theorem~\ref{thm:opt_adaptive_runtime} (below) and Theorem~\ref{thm:simple_runtime} are consistent with Theorems 5 \& 6 and Theorem 3 respectively from \cite{BottcherDN10}, where $\ell = 1$ (in which case $b = 1/4$ and $a=0$).

		\begin{theorem}
			\label{thm:optimal_static_mutation_rate}
			Let $p$ be the static mutation rate of the \oea that minimizes its runtime on the \blockleadingones problem. Let $\lambda$ be the value of $x$ that minimizes the function $g(x) = (e^x - 1)/x^2$ for $x > 0$. Assume $\ell = o(n)$.
			Then $p = (1+o(1))\frac{\lambda}{n}$.
		\end{theorem}
		Note that this is exactly the same optimal static mutation rate as for \leadingones\ problem; see
		\cite{BottcherDN10}.

		\begin{corollary}
			\label{cor:opt_static_runtime}
			Let $T$ be the runtime of the \blockleadingones problem using the optimal static mutation rate given in 
			Theorem~\ref{thm:optimal_static_mutation_rate}. Let $b$ be as in Theorem \ref{thm:simple_runtime},
			and let $\alpha = \min_{x > 0} (e^x - 1)x^{-2} \approx 1.54$. Assume $\ell = o(n)$.
			Then
			\[
			\expectation[T] = (1 + o(1)) \alpha \cdot \frac{b 2^{\ell}}{\ell}n^2.
			\]
			If $\ell = \omega(1)$, then
			\[
			\expectation[T] = (1 + o(1)) \alpha \cdot \frac{2^{\ell}}{\ell^2}n^2.
			\]
		\end{corollary}
		\begin{proof}
			The first part follows from Theorems~\ref{thm:simple_runtime} and \ref{thm:optimal_static_mutation_rate}.
			Also, by Lemma \ref{lem:from_Marcin}, $\ell = \omega(1)$ implies that $b = \ell^{-1}(1 + o(1))$.
		\end{proof}

		\begin{corollary}
			\label{cor:any_growth_rate}
			By an appropriate choice of $\ell$, the expected runtime $T$ of the $(1 + 1)$ EA on \blockleadingones
			using the optimal static mutation rate can have any growth rate that is $\omega(n^2)$ and $2^{o(n)}$. In other words,
			let $h(n)$ be a function such that $h(n) = \omega(n^2)$ and $h(n) = 2^{o(n)}$. Then $\ell$ can be chosen so that
			\[
			\lim_{n \to \infty} \frac{\expectation[T]}{h(n)}  = 1.
			\]
			Further, choosing $\ell$ to be constant, we can make the expected runtime $\Theta(n^2)$, with various choices for the
			hidden constant(s).
		\end{corollary}
		Note that the same result as Corollary \ref{cor:any_growth_rate} is true for the runtime when using a fitness-dependent mutation rate. Also, note that the upper limit of $2^{o(n)}$ is due to our assumption that $\ell = o(n)$; one could assume $\ell = \Theta(n)$, but this is not considered in this paper (apart from \S \ref{sec:needle_problem} where $\ell = n$).
		
		Let $k = \ell \cdot \blockleadingones(x)$, where $x$ is the current individual; so $k$
		denotes the number of bits locked in by the elitist \oea on an $n$-bit problem.

		In order to make sense of Theorem \ref{thm:optimal_mutation_rate}, we need to be able to let $n$ approach infinity (for
		otherwise, $\ell$ and the fitness $m$ are bounded). To do that, note from Lemma \ref{lem:from_needle_to_block}
		below that the expected optimization time of a block does not directly depend on $n$ (but only on $p$, $\ell$,
		and $k$). In other words, the expected optimization time of one block is a function of $p$, $\ell$, and $k$ \emph{alone}
		and not at all on $n$. Hence, we may freely let $m$ approach infinity, and similarly for any appropriate $\ell$.
		\begin{theorem}
			\label{thm:optimal_mutation_rate}
			Let $p_m$ denote the optimal mutation rate to optimize the next block, given the current individual has fitness $m$. Assume $\ell = o(n)$.
			% We have $p(0) = 1/2$. This sentence is false. Would it be hard to show p(0) approaches 1/2 as \ell approaches infinity?
			Then
			\[
			\lim_{\ell \to \infty} \frac{p_m}{\ell^{-1}(\sqrt{1+2/m}-1)} = 1.
			\]
			Also, as $m \to \infty$, for $k = m\ell$ we have
			\[
			p_m = (1+o(1))\frac{1}{k}.
			\]
		\end{theorem}
		
		\begin{theorem}
			\label{thm:opt_adaptive_runtime}
			Let $T$ be the runtime of the \oea on the \blockleadingones problem, where we use the optimal fitness-dependent mutation rate $p_m$. 
			%given in Theorem~\ref{thm:optimal_mutation_rate}. 
			Let $b$ be as in Theorem~\ref{thm:simple_runtime}.
			Assume $\ell = o(n)$.
			Then
			\[
			\expectation[T] = (1 + o(1)) \frac{e}{2} \cdot \frac{b 2^{\ell}}{\ell}n^2.
			\]
		\end{theorem}
		Note that it is actually easier to prove the exact expression in Theorem~\ref{thm:exact_fitness_dep_runtime} than the estimate of the optimal runtime in Theorem~\ref{thm:opt_adaptive_runtime}.
		Also, a consequence of Theorem~\ref{thm:opt_adaptive_runtime} and Corollary~\ref{cor:opt_static_runtime} is that
		the runtime when using the optimal fitness-dependent rate(s) is about 0.88 of the runtime when using the optimal static mutation rate (just like \leadingones). Indeed,
		note that $(e/2)/\alpha \approx 0.88$, where $\alpha$ is as in Corollary \ref{cor:opt_static_runtime}.
		
		\subsection{Estimating the Runtime on One Block}
        \label{sec:runtime_of_one_block}
		
		In this section, we give a relatively simple expression for the amount of time spent on the next, unoptimized block. 
		Already, Proposition \ref{prop:exact_expectation_starting_away_from_ones} gives an exact expression for this,
		but we need to put it in a form that reveals more clearly how $p$ affects its size.
		After doing this in Lemma~\ref{lem:simplified_expectation}, we then show how similar the expected runtime on one plateau in \blockleadingones is to the plateau in \cite{AntipovD21telo}. We then show how to estimate a certain sum that appears in Lemma~\ref{lem:simplified_expectation}.
		
		Let $T_k$ denote the runtime of optimizing the next block after having locked in exactly the first $k$ bits,
		and so assume that the next block is not already optimized. 
		
		\begin{lemma}
			\label{lem:from_needle_to_block}
			Let $T_k$ be as in the previous paragraph, and let $T'$ be as in Proposition \ref{prop:exact_expectation_starting_away_from_ones}.
			Then
			\[\tag{$*$}\label{eq:lift_to_actual_block}
			\expectation[T_k] = \frac{\expectation[T']}{(1-p)^k},
			\]
			where $p$ is the mutation rate used on this next block.
		\end{lemma}
		\begin{proof}
			Any step in which any of the first $k$ bits is flipped will result in an individual of lower fitness,
			which will be discarded. This result follows since the probability that none of the first $k$ bits is flipped is $(1-p)^k$.
		\end{proof}
		
		\begin{lemma}
			\label{lem:simplified_expectation}
			% TODO: maybe rewrite this so that it also has the result for \expectation[T_k'] or add it specifically to Lemma \ref{lem:from_needle_to_block_allowing_0_steps}.
			Let the function $s$ be as in Lemma \ref{lem:from_Marcin}. %, and let $c = c(\ell)$. 
			We have  
			\[
			\expectation[T_k] = \frac{2^\ell}{2^\ell - 1}  \cdot \frac{2^{\ell}}{(1-p)^k}\left[ \frac{b}{p} + a + O(p) \right],
			\]
			where
			\[
			a =  \frac{1}{2} - \frac{1 + s(\ell)}{2^{\ell+1}}, \text{ \quad and \quad}
			b =\frac{s(\ell)}{2^{\ell + 1}}.
			\]
		\end{lemma}
		\begin{proof}
			This follows from Proposition \ref{prop:exact_expectation_starting_away_from_ones}      %Theorem \ref{thm:exact_expectation} 
			and
			Lemma \ref{lem:from_needle_to_block} by using the
			Taylor series for $\ds{\frac{1}{1-(1-2x)^j}}$.  Indeed, the Taylor series gives us this:
			\[
			\frac{1}{1-(1-2x)^j} = \frac{1}{2jx} +  \frac{1}{2} - \frac{1}{2j} + O(x).
			\]
			Let $A = \frac{2^\ell}{2^\ell - 1}$. Using the above Taylor series, 
			Lemma \ref{lem:from_needle_to_block} and Proposition \ref{prop:exact_expectation_starting_away_from_ones}  
			imply that
			\[\begin{aligned}
			\expectation[T_k] &= \frac{A}{(1-p)^k} \sum_{j=1}^\ell \binom{\ell}{j} \left[ \frac{1}{2jp} + \frac{1}{2} - \frac{1}{2j}  + O(p) \right] \\
			&= \frac{A}{(1-p)^k} \left[ \frac{1}{2p}  \sum_{j=1}^\ell \binom{\ell}{j} \frac{1}{j} + \frac{1}{2}\sum_{j=1}^\ell \binom{\ell}{j} 
			- \frac{1}{2} \sum_{j=1}^\ell \binom{\ell}{j} \frac{1}{j} + \sum_{j=1}^\ell \binom{\ell}{j}  O(p) \right] \\
			&= \frac{A}{(1-p)^k} \left[ \frac{s(\ell)}{2p} + \frac{2^\ell - 1}{2} - \frac{s(\ell)}{2} + (2^{\ell} - 1)O(p) \right] \\
			&= \frac{A}{(1-p)^k} \left[ \frac{s(\ell)/2}{p}   + 2^{\ell - 1} - \frac{1 + s(\ell)}{2} + (2^{\ell} - 1)O(p) \right],
			\end{aligned}
			\]
			which gives this result once we factor out $2^{\ell}$.
		\end{proof}
		
		Let $\beta$ be the probability of accepting an offspring with at least one bit flipped in the next unoptimized block. We claim that a consequence of Lemma~\ref{lem:simplified_expectation} is that for large $\ell$ and for $k \geq 1$, roughly speaking,
		$\expectation[T_k]$ is approximately 
		\[\tag{$*$}\label{fraction-approximation}
		% that simultaneously we aren't flipping any of the first k bits
		\frac{\text{effective size of the plateau}}{\beta},
		\]
		which is very similar to the main result of \cite{AntipovD21telo}.  The plateau has exactly $(2^{\ell} - 1)(2^{n - k - \ell})$ elements
		in it, and there are $2^{n - k - \ell}$ elements that improve the fitness. So the effective size of the plateau is $2^{\ell} - 1 \approx 2^\ell$.
		So (\ref{fraction-approximation}) becomes $2^{\ell}/\beta$, and 
		so the value analogous to 
		\cite{AntipovD21telo} would be
		\[
		\frac{2^\ell}{(1-p)^k(1-(1-p)^\ell)},
		\]
		which we now show is what $\expectation[T_k]$ is approximately. We may simplify by applying a Taylor series expansion on part of it:
		\[\tag{$**$}\label{eq:antipov_doerr_result}
		\frac{2^\ell}{(1-p)^k(1-(1-p)^\ell)} = \frac{2^\ell}{(1-p)^k}\left[\frac{1}{\ell p} + \frac{\ell-1}{2\ell} + O(p) \right].
		\]
		Similarly, by Lemma \ref{lem:simplified_expectation} we have
		\[\begin{aligned}
		\expectation[T_k] &\approx \frac{2^{\ell}}{(1-p)^k} \left[\frac{b}{p} + a + O(p) \right] \\
		&= \frac{2^{\ell}}{(1-p)^k} \left[\frac{b'}{\ell p} + a+ O(p) \right],
		\end{aligned}
		\]
		where $b' = b\ell = \ell s(\ell)2^{-\ell-1}$ and $a = 1/2 - 2^{-\ell - 1} - b$. 
		Since $\ell$ is assumed to be large, by Lemma \ref{lem:from_Marcin}, we have $\ell s(\ell)2^{-\ell-1} \approx 1 + c/\ell$ 
		for some $c \approx 1$. We've shown $b' \approx 1 + c/\ell$ for some $c \approx 1$ and a little simplifying shows that
		for large $\ell$, we also have $a \approx 1/2 - 0 - 1/\ell = (\ell - 2)/(2\ell)$. We have thus shown that $\expectation[T_k]$ is approximately
		(\ref{eq:antipov_doerr_result}).
		
		\subsubsection{Approximating a Certain Sum}
		
		In this section, we show how to approximate the sum $s(m)$ defined in Lemma~\ref{lem:from_Marcin}. We need this result because $s(\ell)$ shows up in the key 
		Lemma~\ref{lem:simplified_expectation}.
		
		\begin{lemma}
			\label{lem:from_Marcin}
			Define the functions $s(m)$ and $f(m)$ via
			\[
			s(m) = \sum_{j=1}^m \binom{m}{j} \frac{1}{j}, \text{\quad and \quad } f(m) = \frac{2^{m+1}}{m}.
			\]
			For all constants $c > 1$, for all large $m$, we have
			\[
			1 + \frac{1}{m} \leq \frac{s(m)}{f(m)} \leq 1 + \frac{c}{m}.
			\]
		\end{lemma}
		\begin{proof}
			This result is proved by Propositions \ref{prop:lower_bound_from_Marcin} and \ref{prop:upper_bound_from_Marcin}.
		\end{proof}
		
		To prove Lemma \ref{lem:from_Marcin} in the two propositions below, we first need a few lemmas.
		\begin{lemma}
			\label{lem:alternate_sn}
			Let $s(m)$ be as in Lemma \ref{lem:from_Marcin}. Then
			\[
			s(m) = \sum_{j=1}^m \frac{2^j - 1}{j}.
			\]
		\end{lemma}
		\begin{proof}
			Let $\ds{h(x) = \sum_{j=1}^m \binom{m}{j} \frac{x^j}{j}}$, and so $h(1) = s(m)$. We have that 
			\[
			h'(x) = \sum_{j=1}^m \binom{m}{j} x^{j-1} = \frac{1}{x} \sum_{j=1}^m \binom{m}{j} x^j,
			\]
			which by the binomial theorem equals
			\[
			\frac{1}{x}[(x+1)^m  - 1],
			\]
			which equals $\sum_{j=0}^{m-1}(1+x)^j$ because $\sum_{j=0}^{m-1} y^j = (y^m - 1)/(y-1)$.
			So then, $h'(x) = \sum_{j=0}^{m-1}(1+x)^j$. Using that $h(0) = 0$, integrating gives
			\[
			h(x) = \sum_{j=0}^{m-1} \frac{(1+x)^{j+1} - 1}{j+1} = \sum_{j=1}^m \frac{(1+x)^j - 1}{j},
			\]
			and plugging in $x=1$ finishes the proof of this lemma.
		\end{proof}
		
		\begin{lemma}
			\label{lem:simple_exp_ineq}
			For all integers $m \geq 0$, we have
			\[
			2^{m+2} \geq m(m-1).
			\]
		\end{lemma}
		Lemma \ref{lem:simple_exp_ineq} can be checked for $m = 0, 1, 2, 3$ and can easily be proved by induction for $m \geq 4$.
		
		%\begin{lemma}
		%	\label{lem:natural_log}
		%	For all $m \geq 8$ we have the following about the natural logarithm:
		%	\[
		%	  \ln(m+1) \geq \frac{m+2}{m-2}.
		%	\]
		%\end{lemma}
		%\begin{proof}
		%	Let $m \geq 8$. Then
		%	\[
		%	  \ln(m+1) \geq \ln(9) = 2\ln(3) > 2 > 1 + \frac{4}{m-2} = \frac{m+2}{m-2}.
		%	\]
		%\end{proof}
		
		We next prove the lower bound in Lemma \ref{lem:from_Marcin}:
		\begin{proposition}
			\label{prop:lower_bound_from_Marcin}
			Let $s(m)$ be as in Lemmas \ref{lem:from_Marcin} and \ref{lem:alternate_sn}. For all large $m$,
			\[
			s(m) \geq \frac{2^{m+1}}{m}\left(1 + \frac{1}{m} \right).
			\]
		\end{proposition}
		\begin{proof}
			Let $c$ be such that 
			\[
			s(t) + c \geq \frac{2^{t+1}}{t} \left(1 + \frac{1}{t-1} \right), \text{ for some positive integer } t.
			\]
			%(Note that the $k$ in this proof is a local variable completely separate from the $k$ in this section.) 
			Then we claim that 
			\[\tag{$*$} \label{ineq:exp1}
			s(m) + c \geq \frac{2^{m+1}}{m} \left(1 + \frac{1}{m-1} \right), \text{ for all } m \geq t,
			\]
			which we will prove by induction on $m$. Assume (\ref{ineq:exp1}) holds for some $m \geq t$. We have by 
			Lemma \ref{lem:alternate_sn} that
			\[
			\begin{aligned}
			s(m+1) + c &= s(m) + c + \frac{2^{m+1} - 1}{m+1} \\
			&\geq \frac{2^{m+1}}{m} \left(1 + \frac{1}{m-1} \right) + \frac{2^{m+1} - 1}{m+1}.
			\end{aligned}
			\]
			We just need to show the following:
			\[\tag{$**$}\label{ineq:goal1}
			\frac{2^{m+1}}{m} \left(1 + \frac{1}{m-1} \right) + \frac{2^{m+1} - 1}{m+1} \geq \frac{2^{m+2}}{m+1} \left(1 + \frac{1}{m}\right).
			\]
			But using $1 + 1/(m-1) = m/(m-1)$ and simplifying/rearranging shows that (\ref{ineq:goal1}) is equivalent to
			\[
			\frac{2^{m+1}}{m-1} \geq \frac{2^{m+1} + 1}{m+1} + \frac{2^{m+2}}{m(m+1)},
			\]
			which is equivalent to each of the following inequalities:
			\[
			\frac{1}{m-1} \geq \frac{1}{m+1} + \frac{2}{m(m+1)} + \frac{1}{(m+1)2^{m+1}} 
			\]
			
			\[
			\frac{2}{(m-1)(m+1)} \geq \frac{2}{m(m+1)} + \frac{1}{(m+1)2^{m+1}}
			\]
			
			\[
			\frac{2}{m-1} - \frac{2}{m} \geq \frac{1}{2^{m+1}}
			\]
			
			\[
			\frac{2}{m(m+1)} \geq \frac{1}{2^{m+1}}
			\]
			
			\[
			2^{m+2} \geq m(m-1),
			\]
			which is true by Lemma \ref{lem:simple_exp_ineq} and so proves (\ref{ineq:exp1}).
			
			Let $m \geq t$. Then (\ref{ineq:exp1}) implies that
			\[
			s(m) \geq \frac{2^{m+1}}{m}\left(1 + \frac{1}{m-1} - \frac{cm}{2^{m+1}} \right),
			\]
			and so we will be finished once we show that for sufficiently large $m$ that
			\[
			\frac{1}{m-1} - \frac{cm}{2^{m+1}}  \geq \frac{1}{m}.
			\]
			But this last inequality is equivalent to this:
			\[
			\frac{1}{m(m-1)} \geq \frac{cm}{2^{m+1}},
			\]
			which in turn is equivalent to $ 2^{m+1} \geq cm^2(m-1),$ which does hold for all large $m$.
		\end{proof}
		
		Note that in the proof of the previous lemma, $c=0$ works for $t=7$, and so $s(m) \geq \frac{2^{m+1}}{m}(1 + \frac{1}{m-1})$
		for all $m \geq 7$.
		
		\begin{proposition}
			\label{prop:upper_bound_from_Marcin}
			Let $s(m)$ be as in Lemmas \ref{lem:from_Marcin} and \ref{lem:alternate_sn}. 
			For all $c > 1$, we have that for all large $m$,
			\[
			s(m) \leq \frac{2^{m+1}}{m}\left(1 + \frac{c}{m} \right).
			\]
		\end{proposition}
		\begin{proof}
			Let $c > 1$. %In this proof, we use $k$ as a local variable unrelated to the main $k$ in this section.
			Assume $t$ is the minimum $m$ such that $1 \leq c\cdot\left(\frac{m-2}{m+2} \right)$, which exists
			since $c > 1$ and $\lim_{m \to \infty} \frac{m-2}{m+2}  = 1$. Suppose $c_0$ is such that
			\[
			s(t) - c_0 \leq \frac{2^{t+1}}{t}\left(1 + \frac{c}{t+1} \right).
			\]
			We claim that 
			\[\tag{$*$}\label{ineq:induction_upper_bound}
			s(m) - c_0 \leq \frac{2^{m+1}}{m}\left(1 + \frac{c}{m+1} \right) \text{\quad for } m \geq t.
			\]
			Just like Proposition \ref{prop:lower_bound_from_Marcin}, we again proceed by induction on  $m$. Assume
			(\ref{ineq:induction_upper_bound}) holds for some $m \geq t$. Then by Lemma \ref{lem:alternate_sn},
			\[
			\begin{aligned}
			s(m+1) - c_0 &= s(m) - c_0 + \frac{2^{m+1} -1}{m+1} \\
			&\leq \frac{2^{m+1}}{m}\left(1 + \frac{c}{m+1} \right) + \frac{2^{m+1} - 1}{m+1}.
			\end{aligned}
			\]
			We just need to show the following:  %\frac{}{}
			\[
			\frac{2^{m+1}}{m}\left(1 + \frac{c}{m+1} \right) + \frac{2^{m+1} - 1}{m+1} \leq \frac{2^{m+2}}{m+1} \left(1 + \frac{c}{m+2} \right).
			\]
			So it is sufficient to prove the following:
			\[
			\frac{2^{m+1}}{m}\left(1 + \frac{c}{m+1} \right) + \frac{2^{m+1}}{m+1} \leq \frac{2^{m+2}}{m+1} \left(1 + \frac{c}{m+2} \right).
			\]
			This last inequality is equivalent to
			\[
			\frac{2^{m+1}}{m}\left(1 + \frac{c}{m+1}\right) \leq \frac{2^{m+1}}{m+1} + \frac{c2^{m+2}}{(m+1)(m+2)},
			\]
			and multiplying by $(m+1)/2^{m+1}$, this is equivalent to
			\[
			\frac{m+1}{m}\left(1 + \frac{c}{m+1} \right) \leq 1 + \frac{2c}{m+2}.
			\]
			This is equivalent to
			\[
			1 + \frac{1}{m} + \frac{c}{m} \leq 1 + \frac{2c}{m+2},
			\]
			or
			\[
			1 \leq c\left(\frac{m-2}{m+2} \right),
			\]
			which is true because $m \geq t$ and $1 \leq c\cdot\left(\frac{t-2}{t+2} \right)$.
			We have thus proved (\ref{ineq:induction_upper_bound}).
			
			Therefore,
			\[
			s(m) \leq \frac{2^{m+1}}{m}\left(1 + \frac{c}{m+1} + \frac{c_0 m}{2^{m+1}}\right).
			\]
			We will be finished with the proof of the current result once we show that for sufficiently large $m$ that
			\[
			\frac{c}{m+1} + \frac{c_0m}{2^{m+1}} \leq \frac{c}{m},
			\]
			but this last inequality is equivalent to
			\[
			\frac{c_0 m}{2^{m+1}} \leq \frac{c}{m(m+1)},
			\]
			which is equivalent to $c_0m^2(m+1) \leq c 2^{m+1}$, which is true for all large $m$.
		\end{proof}
		
		\subsection{The Optimal Static Mutation Rate}
        \label{sec:opt_static_mutation_rate}
		
		In this section, we prove Theorems~\ref{thm:exact_fitness_dep_runtime}, \ref{thm:simple_runtime}, and \ref{thm:optimal_static_mutation_rate}
		and Corollary~\ref{cor:any_growth_rate}.
		To begin, we need a couple of lemmas, but first we mention a subtlety.
		\begin{definition}
			\label{def:T_k_prime}
			Suppose $N$ is the exact number of steps it takes to optimize the first $k$ bits.
			Define $T_k'$ as the number of additional steps beyond $N$ until the the next block
			is optimized \textbf{which might be 0 steps} since the moment the first $k$ bits are all 1's, then
			it might happen that simultaneously, the next $\ell$ bits also happen to all be 1.
		\end{definition}
		
		\begin{lemma}
			\label{lem:from_needle_to_block_allowing_0_steps}
			Let $T_k'$ be as  in Definition \ref{def:T_k_prime}.
			Let $T_k$ be as in Lemma \ref{lem:from_needle_to_block}, 
			and
			let $T$ be as in Theorem \ref{thm:exact_expectation}. Then
			\[
			\expectation[T_k'] =  \frac{\expectation[T]}{(1-p)^k}, \text{\; and so \;}  \expectation[T_k'] = \frac{2^\ell - 1}{2^\ell} \expectation[T_k]
			= \frac{2^{\ell}}{(1-p)^k}\left[ \frac{b}{p} + a + O(p) \right],
			\]
			where $a$ and $b$ are from Lemma \ref{lem:simplified_expectation}.
		\end{lemma}
		A proof of the first part of Lemma \ref{lem:from_needle_to_block_allowing_0_steps} is almost identical to that of 
		Lemma \ref{lem:from_needle_to_block}, and the rest follows from Lemma~\ref{lem:simplified_expectation}.

        \begin{proof}[Proof of Theorem \ref{thm:exact_fitness_dep_runtime}]
        Let $T_k'$ be as  in Definition \ref{def:T_k_prime}. For fitness $m$, by the first part of Lemma~\ref{lem:from_needle_to_block_allowing_0_steps} we have $\expectation[T'_{m\ell}] = \expectation[T]/(1-p_m)^{m\ell}$. Note that $T = \sum_{m=0}^{n/\ell-1} T_{m\ell}'$,
			and so by linearity of expectation, we are done with the first part by summing the above expression for $\expectation[T'_{m\ell}]$ and using Thoerem~\ref{thm:exact_expectation}. To prove the formula for the static mutation rate $p$, just use that $\sum_{m=0}^{k-1} x^m = (x^{k} - 1)/(x-1)$.
        \end{proof}
		
		\begin{lemma}
			\label{lem:entire_runtime}
			Let $a$ and $b$ be as in Lemma \ref{lem:simplified_expectation}, namely that
			\[
			a =  \frac{1}{2} - \frac{1 + s(\ell)}{2^{\ell+1}}, \text{ \quad and \quad}
			b =\frac{s(\ell)}{2^{\ell + 1}}.
			\]
			The expected runtime, $T$, of the \blockleadingones problem when using a constant mutation rate of $p$ is
			\[
			\expectation[T] = 2^{\ell} \left(\frac{b}{p} + a + O(p) \right)\frac{(1-p)^{-n+\ell} - (1-p)^\ell}{1-(1-p)^\ell}.
			\]
		\end{lemma}
		\begin{proof}
			Let $T_k'$ be as in Definition \ref{def:T_k_prime}. 
			Then $T = \sum_{m=0}^{n/\ell-1} T_{m\ell}'$.
			So by linearity of expectation and Lemmas \ref{lem:simplified_expectation} and \ref{lem:from_needle_to_block_allowing_0_steps}, we have
			\[
			\begin{aligned}
			\expectation[T] &= \sum_{m = 0}^{n/\ell-1} \expectation[T_{m\ell}'] \\
			&= \sum_{m = 0}^{n/\ell-1}  \frac{2^{\ell}}{(1-p)^{m\ell}} \left[\frac{b}{p} + a + O(p) \right] \\
			&= 2^{\ell}\left[\frac{b}{p} + a + O(p) \right] \sum_{m = 0}^{n/\ell-1}  \left(\frac{1}{1-p} \right)^{m \ell} \\
			%&= 2^{\ell}\left[\frac{b}{p} + a + O(p) \right] \frac{1-(1/(1-p))^{\ell(n/\ell + 1)}}{1 - (1/(1-p))^\ell} \\
			&= 2^{\ell} \left(\frac{b}{p} + a + O(p) \right)\frac{(1-p)^{-n+\ell} - (1-p)^\ell}{1-(1-p)^\ell}.
			\end{aligned}
			\]
		\end{proof}
		
		\begin{proof}[Proof of Theorem \ref{thm:simple_runtime}]
			Using Taylor series,
        	\[
        	\left(1 - \frac{c}{n} \right)^\ell 
        	= 1  - \frac{\ell c}{n} + O\left( \frac{\ell^2}{n^2}c \right).
        	\]
        	Since $\ell = o(n)$, we have $(1-c/n)^\ell \to 1$. So, $(1-p)^{-n+\ell} = (1-c/n)^{-n+\ell} = (1+o(1))e^c$.
			So by Lemma \ref{lem:entire_runtime}, and we have 
			\[
    	\begin{aligned}
    	\expectation[T] &= 2^{\ell} \left(\frac{bn}{c} + a + O(c/n) \right)\frac{(1-c/n)^{-n+\ell} - (1-c/n)^\ell}{1- (1-c/n)^\ell} \\
    	&= 2^{\ell} \left(\frac{bn}{c} + a + O\left(\frac{c}{n}\right)\right) 
    	\frac{(1+o(1))e^c - 1 + \ell c /n + O(\ell^2 c^2 /n^2)}{1 - 1 + \ell c /n + O(\ell^2 c^2 /n^2)} \\
    	&= \frac{n 2^{\ell}}{\ell} \left(\frac{bn}{c^2} + \frac{a}{c} + O\left(\frac{1}{n}\right)\right) 
    	\frac{(1+o(1))e^c - 1 + \ell c /n + O(\ell^2 c^2 /n^2)}{1 + O(\ell c /n)}.
    	\end{aligned}
    	\]
    	Let $\ds{g(c, n) = \frac{n 2^{\ell}}{\ell} \left(\frac{bn}{c^2} + \frac{a}{c} + O\left(\frac{1}{n}\right)\right) }$. So then,
    	\[
    	\begin{aligned}
    	\expectation[T] &= g(c, n) \left(\frac{(1+o(1))e^c - 1 + \ell c /n }{1 + O(\ell c /n)} + \frac{O(\ell^2 c^2/n^2)}{1 + O(\ell c /n)} \right) \\
    	&= g(c, n)\left(\frac{(1+o(1))e^c - 1 + \ell c /n }{1 + O(\ell c /n)} + o(1)\right) \\
    	&= (1 + o(1)) g(c, n) \frac{(1+o(1))e^c - 1 + \ell c /n }{1 + O(\ell c /n)}.
    	\end{aligned}
    	\]
    	But since $\ds{\frac{1}{1 + O(\ell c /n)} = 1 + o(1)}$,  we have
    	\[
    	\expectation[T] = (1 +o(1)) g(c, n)((1+o(1))e^c - 1 + \ell c /n).
    	\]
    	But $\ds{g(c, n) = \frac{n 2^{\ell}}{\ell} \left(\frac{bn}{c^2} + \frac{a}{c} + o(1)\right)  }$, and so
    	\[
    	  \expectation[T] = (1 + o(1)) \frac{n 2^{\ell}}{\ell} \left(\frac{bn}{c^2} + \frac{a}{c} + o(1)\right) ((1+o(1))e^c - 1 + \ell c /n),
    	\]
    	which equals this:
    	\[
    	(1 + o(1)) \frac{n 2^{\ell}}{\ell} \left(\frac{bn}{c^2} + \frac{a}{c}\right) ((1+o(1))e^c - 1 + \ell c /n).
    	\]
    	But since $\ell = o(n)$, we have $\ell c/n = o(1)$, and so we get
    	\[
    	\begin{aligned}
    	\expectation[T] &= (1 + o(1)) \frac{n 2^{\ell}}{\ell} \left(\frac{bn}{c^2} + \frac{a}{c}\right) ((1+o(1))e^c - 1 + o(1)) \\
    	&= (1 + o(1)) \frac{n 2^{\ell}}{\ell} \left(\frac{bn}{c^2} + \frac{a}{c}\right) (e^c - 1)
    	\end{aligned}
    	\]
		\end{proof}

		\begin{lemma}
			\label{lem:simplifying_fn_for_static_mutation}
			Let $a$ and $b$ be as in Lemmas \ref{lem:simplified_expectation} and \ref{lem:entire_runtime}.  
			Define $g_n(x)$ and $h_n(x)$ as follows:
			\[       % TODO: maybe change g_n(x) once I change a and b
			g_n(x) = \frac{2^{\ell} b n^2}{\ell x^2} (e^x - 1), \text{ \quad and}
			\]
			\[
			h_n(x) = \frac{n 2^{\ell}}{\ell} \left(\frac{bn}{x^2} + \frac{a}{x} \right) (e^x - 1 ).
			\]
			Then we have
			\[
			\lim_{n \to \infty} \frac{h_n(x)}{g_n(x)} = 1.
			\]
		\end{lemma}
		\begin{proof}
			This follows from very basic Calculus, knowing that $a$ and $b$ do not depend on $n$. Specifically, 
			\[
			\begin{aligned}
			\frac{h_n(x)}{g_n(x)} = \frac{(n 2^{\ell}/\ell ) \left(bn/x^2 + a/x \right) (e^x - 1)}{(2^{\ell} b n^2/(\ell x^2)) (e^x - 1)} 
			= \frac{bn/x^2 + a/x}{b n/x^2}
			&= \frac{b/x^2 + a/(nx)}{b/x^2} \\
			&\to \frac{b/x^2}{b/x^2} 
			= 1.
			\end{aligned}
			\]
		\end{proof}
		
		\begin{proof}[Proof of Theorem \ref{thm:optimal_static_mutation_rate}]
			Let $h_n(x)$ and $g_n(x)$ be as in Lemma~\ref{lem:simplifying_fn_for_static_mutation}.
			By Theorem \ref{thm:simple_runtime}, we need only to optimize the function(s) $h_n(x)$. 
			But by Lemma~\ref{lem:simplifying_fn_for_static_mutation},
			we may use Lemma~\ref{lem:from_gk_to_hk} so that we need only optimize $g_n(x) = \frac{2^{\ell} b n^2}{\ell x^2} (e^x - 1)$.
			To optimize this, we need only optimize this function:
			\[
			g(x) = \frac{e^x - 1}{x^2}.
			\]
		\end{proof}

		\begin{proof}[Proof of Corollary \ref{cor:any_growth_rate}]
			As we will see, this follows from Corollary \ref{cor:opt_static_runtime}. The last statement of Corollary \ref{cor:any_growth_rate}
			should be clear from Corollary \ref{cor:opt_static_runtime}. So assume $\ell = \omega(1)$.
			
			By Corollary \ref{cor:opt_static_runtime},
			\[
			\expectation[T] = (1 + o(1)) \frac{\alpha 2^{\ell}}{\ell^2} \cdot n^2,
			\]
			where $\alpha \approx 1.54$. Let $h(n)$ be a function such that $h(n) = \omega(n^2)$ and $h(n) = 2^{o(n)}$.
			Define $g(n)$ as $g(n) = h(n)/n^2$. Then $g(n) = \omega(1)$ and $g(n) = 2^{o(n)}$. We will be done once
			we show that there is an appropriate choice of $\ell$ such that $\alpha 2^{\ell}/\ell^2 = g(n)$. 
			%Note that  since we are now assuming that $\ell = \omega(1)$, that implies that \frac{c 2^{\ell}}{\ell^2} is \omega(1).
			
			Define $f(m)$ as
			\[
			f(m) = \frac{\alpha 2^{m}}{m^2},
			\]
			and note that since $f(m)$ is an increasing function for $m > 2/\ln(2)$, then it has an inverse on that interval,
			which we denote $f^{-1}$. Define $\ell$ as
			\[\tag{$*$}\label{eq:choice_of_ell}
			\ell = f^{-1}(g(n)).
			\]
			Then by definition of $f$, we have $f(\ell) = \alpha 2^{\ell}/\ell^2$ and by definition of the inverse, we have $f(\ell) = g(n)$, 
			and hence we get $\alpha 2^{\ell}/\ell^2 = g(n)$, but we are not quite done yet.
			
			Recall that we have proved Corollary~\ref{cor:opt_static_runtime} under the standing 
			assumption of this paper that $\ell = o(n)$. Hence, all we need to do now is show that the choice (\ref{eq:choice_of_ell})
			is an appropriate choice of $\ell$, namely that for such $\ell$, we have $\ell = o(n)$. Indeed, since $g(n) = 2^{o(n)}$, note 
			that (\ref{eq:choice_of_ell}) implies that $f(\ell) = 2^{o(n)}$, or $c 2^\ell/\ell^2 = 2^{o(n)}$, which is the same as
			saying $2^\ell/\ell^2 = 2^{o(n)}$, but $2^\ell/\ell^2 = 2^{\ell - 2\log_2(\ell)}$, and so 
			$2^{\ell - 2\log_2(\ell)} = 2^{o(n)}$, which implies that $\ell - 2\log_2(\ell) = o(n)$. But $\ell - 2\log_2(\ell) > \ell/2$ for
			all large $\ell$ and so we get $\ell/2 = o(n)$, which is the same as saying $\ell = o(n)$, which completes this proof. 
		\end{proof}
	
	   \subsection{Optimal Fitness-dependent Mutation Rate}
		\label{sec:opt_fitness_dep_mutation_rate}
  
		\subsubsection{Erasing the $O(p)$ Term in Lemma \ref{lem:simplified_expectation}}
		The main point of this subsection is to show that when minimizing $\expectation[T_k]$ based on the choice of $p$, we can just erase the $O(p)$ term and carry on. The other purpose of this section is the easier job of minimizing the simpler resulting function(s).
		
		\begin{definition}
			\label{def:f_p_tf_and_tp}
			Let $a$ and $b$ be as in Lemma \ref{lem:simplified_expectation}: using $s$ from Lemma \ref{lem:from_Marcin},
			\[
			a =  \frac{1}{2} - \frac{1 + s(\ell)}{2^{\ell+1}}, \text{ \quad and \quad}
			b =\frac{s(\ell)}{2^{\ell + 1}}.
			\]
			Define $\fh$, $\tf$, and $g$ as follows: 
			\[
			\begin{aligned}
			\fh(x, k) &= \frac{1}{(1-x/k)^k}\left[ \frac{bk}{x} + a + O\left(\frac{x}{k}\right) \right] \\
			\tf(x, k) &= e^x\left[ \frac{bk}{x} + a\right] \\
			g(x) &= e^x \left[ \frac{m}{x} + \frac{1}{2} \right]
			\end{aligned}
			\] 
			where $m \geq 1$ is an integer and the constant hidden in the $O(x)$ notation is independent of $k$, and define
			$\ph(k)$, $\tp(k)$, and  $\rho$ by
			\[\begin{aligned}
			\ph(k) &= \argmin_{x \in (0, \infty)} \fh(x, k), \text{ and } \\
			\tp(k) &= \argmin_{x \in (0, \infty)} \tf(x, k), \text{ and } \\
			\rho &= \argmin_{x \in (0, \infty)}  g(x).
			\end{aligned}
			\]
		\end{definition}

		%It turns out that for large $k$, to optimize $f$ we need only optimize $\tf$ and $g$.
		
		\begin{lemma}
			\label{lem:f_and_tf_are_asymptotic}
			Let $\fh$ and $\tf$ be as in Definition \ref{def:f_p_tf_and_tp}.
			For all (real) $x$ we have
			\[
			\lim_{k \to \infty} \frac{\fh(x, k)}{\tf(x, k)} = 1.
			\]
			%where the convergence is uniform on any closed interval $[x_0, x_1]$ where we assume $0 < x_0 < x_1$.
		\end{lemma}
		\begin{proof}
			This follows from basic Caluclus. Indeed, choose any $x$. Then 
			\[
			\lim_{k \to \infty} \frac{1}{(1-x/k)^k} = \frac{1}{e^{-x}} = e^x,
			\]
			and also,
			\[
			\lim_{k \to \infty} \frac{(b/x)k + a + O(x/k)}{(b/x)k + a} = 1.
			\]
			% no longer todo: to finish the proof about it being uniform on closed intervals (which I no longer think I need)
		\end{proof}
		
		\begin{lemma}
			\label{lem:f_and_g_are_asymptotic}
			Let $\fh$ and $g$ be as in Definition \ref{def:f_p_tf_and_tp}. Let $k = m\ell$, where $m$ is some
			fixed positive integer. Then
			\[
			\lim_{\ell \to \infty} \frac{\fh(x, m\ell)}{g(x)} = 1.
			\]
		\end{lemma}
		\begin{proof}
			Recall the definition of $a$ and $b$ from Definition \ref{def:f_p_tf_and_tp}. By Lemma \ref{lem:from_Marcin},
			we have
			\[\tag{$*$}\label{eq:s_limit}
			\lim_{\ell \to \infty} \frac{s(\ell)}{2^{\ell+1}/\ell} = 1.
			\]
			Therefore, $bk = b m \ell = \ds{\frac{m \ell s(\ell)}{2^{\ell + 1}}}$ implies that
			\[
			\lim_{\ell \to \infty } bk = \lim_{\ell \to \infty } b m \ell = \lim_{\ell \to \infty } \frac{m \ell s(\ell)}{2^{\ell + 1}} = m.
			\]
			
			Also, (\ref{eq:s_limit}) implies that $s(\ell)2^{-\ell-1} \to 0$, which implies that  
			\[
			a = \frac{1}{2} - (1 + s(\ell))2^{-\ell-1} \to \frac{1}{2}, \text{ as } \ell \to \infty.
			\]
			
			Using this and the limit of $bm\ell$, together with $\ds{\lim_{\ell \to \infty} (1-x/(m\ell))^{-m\ell} = e^x}$,  we get
			\[
			\begin{aligned}
			\lim_{\ell \to \infty} \frac{\fh(x, m\ell)}{g(x)}  &= \lim_{\ell \to \infty} \frac{(1-x/(m\ell))^{-m\ell}
				\left[ bm\ell/x + a + O\left(x/(m\ell)\right) \right] }{e^x \left[ m/x + 1/2 \right]} \\
			&= \frac{m/x + 1/2}{m/x + 1/2} \\
			&= 1.
			\end{aligned}
			\]
		\end{proof}
		
		\begin{lemma}
			\label{lem:root_tp}
			Let $a$, $b$, $\tf$ and $\tp(k)$ be as in Definition \ref{def:f_p_tf_and_tp}. If $k > 0$ then
			\[
			\tp(k) = \frac{bk}{2a}\left(-1 + \sqrt{1 + \frac{4a}{bk}} \right).
			\]
			Further, given any fixed $\ell$, we have
			\[
			\lim_{k \to \infty} \tp(k) = 1.
			\]
		\end{lemma}
		\begin{proof}
			This follows from standard Calculus.
			Let $\beta = bk$. So $\tf = e^x(\beta/x + a)$. We have
			\[
			\begin{aligned}
			\tf'(x) &= e^x\left(\frac{\beta}{x} + a - \frac{\beta}{x^2} \right) \\
			&= \frac{e^x}{x^2} \left(a x^2 + \beta x - \beta\right).
			\end{aligned}
			\]
			We have then that $\tf$ is minimized at the largest root of $ax^2 + \beta x - \beta$. (This is because $\tf'$ is negative before that
			root and positive afterwards.) Hence, the quadratic formula gives
			\[
			\begin{aligned}
			\tp(k) &= \frac{-\beta + \sqrt{\beta^2 + 4a\beta}}{2a} \\
			&= \frac{\beta}{2a} \left(-1 + \sqrt{1 + 4a/\beta} \right) \\
			&= \frac{bk}{2a}\left(-1 + \sqrt{1 + \frac{4a}{bk}} \right)
			\end{aligned}
			\]
			which proves the first part of this result.
			
			Next, since $\ell$ is constant (with respect to $k$), then that implies that
			$a$ and $b$ are constant. We use that the tangent line to $\sqrt{1 + 2x}$ at $x=0$ is $1+x$.  So since
			$4a/(bk) \to 0$ as $k \to \infty$, we have
			\[
			\begin{aligned}
			\lim_{k \to \infty} \tp(k) &= \lim_{k \to \infty}  \frac{bk}{2a}\left(-1 + \sqrt{1 + \frac{4a}{bk}} \right) \\
			&= \lim_{k \to \infty}  \frac{bk}{2a} \left(-1 + 1 + \frac{2a}{bk} \right) \\
			&= 1.
			\end{aligned}
			\]
		\end{proof}
		
		\begin{lemma}
			\label{lem:value_of_rho}
			Let $g$ and $\rho$ be as in Definition \ref{def:f_p_tf_and_tp}. If $m > 0$ then
			\[
			\rho = m(-1 + \sqrt{1 + 2/m}).
			\]
		\end{lemma}
		\begin{proof}
			This also follows from basic Calculus and is very similar to the proof of the first part of
			Lemma \ref{lem:root_tp}. 
		\end{proof}
		
		The idea of the following lemma is that for appropriate functions $h_k$ and $g_k$, in order to find where $h_k$
		achieves its minimum, we need only find where $g_k$ achieves its minimum. We will use  Lemma \ref{lem:from_gk_to_hk} later on
		with $g_k(x) = \tf(x, k)$.
		\begin{lemma}
			\label{lem:from_gk_to_hk}
			%Pick real numbers $\alpha$ and $\rho$. % at one time: $\alpha_k$ and $\rho_k$
			Suppose $h_k$ and $g_k$ are functions such that for all $x$,
			\[
			\lim_{k \to \infty} \frac{h_k(x)}{g_k(x)} = 1.
			\]
			%where the convergence is uniform in some interval $I$ with $\alpha_k, \rho_k \in I$ for all $k$.
			Let $\alpha$ and $\rho$ be such that for some $\delta$ with $0 < \delta < 1$ we have
			\[
			\frac{g_k(\rho)}{g_k(\alpha)} \leq \delta, \text{\quad for all large } k.
			\]
			Then we have that for all large $k$, 
			\[
			h_k(\rho) < h_k(\alpha).
			\]
		\end{lemma}
		\begin{proof}
			We have
			\[
			\frac{h_k(\rho)}{h_k(\alpha)}  = \frac{h_k(\rho)}{g_k(\rho)} \cdot \frac{g_k(\rho)}{g_k(\alpha)} \cdot \frac{g_k(\alpha)}{h_k(\alpha)},
			\]
			the first and last fractions of which approach 1 while the middle fraction is bounded by $\delta$ for all large $k$. Hence,
			for all large $k$, we have $\ds{\frac{h_k(\rho)}{h_k(\alpha)} < 1}$, proving the result.
		\end{proof}
		
		\begin{lemma}
			\label{lem:bounding_ratio_of_tf}
			Let $\tf$ be as in Definition \ref{def:f_p_tf_and_tp}.
			Fix any $\eps > 0$. Then there exists a $\delta$ (depending on $\eps$) with $0 < \delta < 1$ such that for all large $k$,
			\[
			\begin{aligned}
			\frac{\tf(1, k)}{\tf(1 + \eps, k)} &\leq \delta \text{\quad and} \\
			\frac{\tf(1, k)}{\tf(1 - \eps, k)} &\leq \delta.
			\end{aligned}
			\]
		\end{lemma}
		\begin{proof}
			We prove the first inequality, as the second one is almost identical. We have that as $k \to \infty$,
			\[
			\frac{\tf(1, k)}{\tf(1 + \eps, k)} = \frac{e(bk + a)}{e^{1+\eps}(bk/(1+\eps) + a)} \longrightarrow \frac{eb}{e^{1+\eps}b/(1+\eps)} = (1+\eps)e^{-\eps}.
			\]
			But the function $R(x) = (1+x)e^{-x}$ has a unique global maximum at $x=0$ with $R(0) = 1$. That proves the first inequality.
			After finding a $\delta_0$ for the first inequality and a $\delta_1$ for the second one, we can define $\delta$ as the max of $\delta_0$
			and $\delta_1$.
		\end{proof}
		
		%\subsubsection{The optimal fitness-dependent mutation rate and resulting runtime}
		
		%The first purpose of this section is to bring together the lemmas proved above in order to prove Theorem \ref{thm:optimal_mutation_rate}. 
		%The other purpose is to show what the resulting runtime is when using that optimal fitness-dependent mutation rate.
		\subsubsection{Proving Theorems \ref{thm:optimal_mutation_rate} and \ref{thm:opt_adaptive_runtime}}
  
		\begin{proof}[Proof of Theorem \ref{thm:optimal_mutation_rate}]
            Let $p(k)$ denote the optimal mutation rate to optimize the next block, given the current individual has fitness $m$ and $k = m\ell$.
            So $p(k) = p_m$, and $k > 0$ since $m > 0$.
            
			By Lemma~\ref{lem:simplified_expectation} and plugging in $x/k$ for $p$, to minimize $\expectation[T_k]$, we need only minimize $\fh(x, k)$, and what is more, 
			\[\tag{$*$}\label{eq:pk_to_phk}
			p(k) = \frac{\ph(k)}{k}.
			\]
			
			First assume $\ell$ is fixed. By Lemmas \ref{lem:f_and_tf_are_asymptotic} and \ref{lem:bounding_ratio_of_tf} and the
			second part of Lemma \ref{lem:root_tp}, we may use Lemma \ref{lem:from_gk_to_hk} to get that for all $\eps > 0$,
			\[
			\begin{aligned}
			\fh(1, k) &\leq \fh(1+ \eps, k) \text{\quad for all large } k, \text{ and} \\
			\fh(1, k) &\leq \fh(1- \eps, k) \text{\quad for all large } k.
			\end{aligned}
			\] 
			Therefore, $\ds{\lim_{k \to \infty} \ph(k) = 1}$, and so (\ref{eq:pk_to_phk}) gives that 
			\[
			\lim_{k \to \infty} \frac{p(k)}{1/k} = \lim_{k \to \infty} k p(k) = \lim_{k \to \infty} \ph(k) = 1.
			\]
			If $m \to \infty$, then $k \to \infty$. Thus $p_m = p(k) = (1+o(1))\cdot 1/k$.
			
			Next, let $k = m\ell$ for some positive integer $m$, and let $\ell \to \infty$. Let $g$ and $\rho$ be as in Definition \ref{def:f_p_tf_and_tp}.
			By Lemma  \ref{lem:value_of_rho}, we have
			\[
			\rho = m(-1 + \sqrt{1 + 2/m}).
			\]
			Let $\eps > 0$. Since $g$ is a single, fixed function, the unique minimum of $g$ at $\rho$ implies that there is a $\delta$ with
			$0 < \delta < 1$ such that
			\[
			\begin{aligned}
			\frac{g(\rho)}{g(\rho + \eps)} &\leq \delta, \text{ and} \\
			\frac{g(\rho)}{g(\rho - \eps)} &\leq \delta.
			\end{aligned}
			\]
			Indeed, we can take $\delta = \max\left(\frac{g(\rho)}{g(\rho + \eps)}, \frac{g(\rho)}{g(\rho - \eps)}\right)$. These inequalities,
			together with Lemma \ref{lem:f_and_g_are_asymptotic} imply by Lemma \ref{lem:from_gk_to_hk} (where we take each
			function $g_k$ to be $g$) that
			\[
			\begin{aligned}
			\fh(\rho, m\ell) &\leq \fh(\rho + \eps, m\ell), \text{ for all large } \ell, \text{ and }\\
			\fh(\rho, m\ell) &\leq \fh(\rho - \eps, m\ell), \text{ for all large } \ell.
			\end{aligned}
			\]
			Therefore, $\ds{\lim_{\ell \to \infty} \ph(m\ell) = \rho}$. Also, (\ref{eq:pk_to_phk}) gives us that
			\[
			p_m = p(k) = p(m\ell) = \frac{\ph(m\ell)}{m\ell}.
			\]
			This last equality and $\ds{\lim_{\ell \to \infty} \ph(m\ell) = \rho}$ imply that
			\[ 
			\lim_{\ell \to \infty} \frac{p_m}{\ell^{-1}(\sqrt{1+2/m}-1)} = 
			\lim_{\ell \to \infty} \frac{\ph(m\ell)/(m\ell)}{\ell^{-1}(\sqrt{1+2/m}-1)} = \lim_{\ell \to \infty} \frac{\ph(m\ell)}{\rho} = 1.
			\]
		\end{proof}

		We next work on what the total expected runtime is on the entire \blockleadingones problem when 
		using the optimal fitness-dependent mutation rate. 
		
		\begin{lemma}
			\label{lem:min_runtime_of_block}
			Let $a$, $b$, and $\tf$ be as in Definition~\ref{def:f_p_tf_and_tp}, and let $T_k'$ be as in Definition~\ref{def:T_k_prime}. 
			Then taking asymptotics as $k \to \infty$,
			\[
			\min_{x} \tf(x, k) = (1 + o(1)) e(bk + a).
			\]
			Consequently, taking the minimum over (fitness-dependent) mutation rates $p$,
			\[
			\min_{p} \expectation[T_k'] = (1 + o(1)) 2^{\ell} e (bk + a).
			\]
		\end{lemma}
		\begin{proof}
			Recall $\tp(k)$ from Definition \ref{def:f_p_tf_and_tp}. Lemma \ref{lem:root_tp} says
			$\lim_{k \to \infty}\tp(k) = 1$, and hence $\tp(k) = 1 + o(1)$. Therefore,
			\[
			\begin{aligned}
			\min_{x} \tf(x, k) = \tf(\tp(k), k) = \tf(1 + o(1), k) &= e^{1+o(1)}(bk/(1+o(1)) + a) \\
			&= (1 + o(1)) e(bk + a). \\
			%&= (1 + o(1)) \tf(1, k).
			\end{aligned}
			\]
			By Lemma \ref{lem:f_and_tf_are_asymptotic}, we have that
			\[
			\min_{x} \fh(x, k) = \min_{x} \tf(x, k), %= (1 + o(1)) \fh(1, k) = (1 + o(1)) \tf(1, k),
			\]
			and so
			\[
			\min_{x} \fh(x, k) = (1 + o(1)) e(bk + a).
			\]
			By Lemma \ref{lem:from_needle_to_block_allowing_0_steps}, we have
			\[
			\min_{p} \expectation[T_k'] = 2^{\ell} \min_{x} \fh(x, k) = (1 + o(1)) 2^{\ell} e (bk + a).
			\]
		\end{proof}

		\begin{proof}[Proof of Theorem \ref{thm:opt_adaptive_runtime}]
			Let $T_k'$ be as in Definition~\ref{def:T_k_prime}, but here, $k = m\ell$ for non-negative integers $m$. 
			Then $T = \sum_{m=0}^{n/\ell-1} T_{m\ell}'$, using $T_k'$ from Definition~\ref{def:T_k_prime}.
			
			So by linearity of expectation and Lemma \ref{lem:min_runtime_of_block}, we have
			\[
			\begin{aligned}
			\expectation[T] &= \sum_{m = 0}^{n/\ell-1} \expectation[T_{m\ell}']\\
			&= \sum_{m = 0}^{n/\ell-1}  (1+o(1))2^{\ell} e (b m \ell + a)\\
			&= \sum_{m = \log(n/\ell)}^{n/\ell-1}  (1+o(1))2^{\ell} e (b m \ell + a) + O(\log(n/\ell)) 2^{\ell} e (b \log(n/\ell) \ell + a) 
			\end{aligned}
			\]
			Let $S$ be the term on the right: $S = O(\log(n/\ell)) 2^{\ell} e (b \log(n/\ell) \ell + a)$, and let 
			$\ds{M = \sum_{m = \log(n/\ell)}^{n/\ell-1} m = \frac{(n/\ell-1) (n/\ell)}{2} - \frac{\log(n/\ell)(\log(n/\ell)+1)}{2} }$. So then, we
			have that
			\[
			\begin{aligned}
			\expectation[T] &= (1 + o(1))eb \ell 2^\ell M +  (1 + o(1))ea 2^\ell  O(n/\ell) + S.
			\end{aligned}
			\]
			Notice that $M = n^2/(2\ell^2) + O(n/\ell)$, and so
			\[
			\expectation[T] = (1 + o(1))\frac{e}{2} \cdot \frac{b 2^{\ell} n^2}{\ell} + eb\ell 2^\ell O(n/\ell) + ea 2^\ell  O(n/\ell) + S,
			\]
			and each of the right-most three terms grows asymptotically slower than the first one and so can be absorbed
			in the $(1 + o(1))$ to get our result:
			\[
			\expectation[T] = (1 + o(1))\frac{e}{2} \cdot \frac{b 2^{\ell} n^2}{\ell}.
			\]
		\end{proof}

		\section{Conclusion}
		
		In this work, we proposed a general method to analyze the time EAs need in order to leave plateaus. Using arguments from discrete Fourier analysis, we obtained exact expressions for these times. Naturally, our method is restricted to plateaus with certain symmetry properties, and this restriction is inherent to discrete Fourier analysis.
		
		In this first work using this method, we restricted ourselves to the \oea with general mutation rate. We are optimistic that our method can also be applied to other simple single-trajectory search heuristics. What is a more interesting question for future research is how EAs with nontrivial population sizes can be analyzed. We note that in this direction, so far only the results~\cite{DoerrK13cec,Eremeev20} exist, which both cannot determine the leading constant of the runtime. In~\cite{NimwegenC01}, a precise bound is stated, but it relies on the unproven assumption ``we can
		assume that in each generation there is an equal and independent probability that epoch $n$
		will end by creating a fitness $n+1$ string that spreads through the population'' [page~92]. 
		Consequently, how to prove a precise runtime estimate for the \oplea optimizing the \needle problem, is clearly a question that waits to be answered. 
		
		A second obvious direction for future work is to investigate how other important insights obtained previously on the \leadingones benchmark extend to the \blockleadingones problem. One particularly interesting topic could be the recent works on hyperheuristics. Since, as shown in this work, the \blockleadingones benchmark contains instances from a broad range of runtimes, it would be interesting to see if the hyperheuristics that show an excellent performance on the \leadingones problem keep their good performance also on the broader \blockleadingones benchmark, where longer times without an improvement must be tolerated.
		
		\section*{Acknowledgments}
		We would like to thank Marcin Mazur for help with the statement and proof of Lemma \ref{lem:from_Marcin}. We also would like to thank the reviewers for their helpful comments, in particular, pointing us to several previous works we were not aware of. This work was supported by a public grant as part of the
		Investissements d'avenir project, reference ANR-11-LABX-0056-LMH,
		LabEx LMH.

		\bibliographystyle{alpha}
		%\bibliography{alles_ea_master,ich_master,rest}
		\newcommand{\etalchar}[1]{$^{#1}$}

	}%end sloppy
\end{document}